\documentclass{article}
\usepackage{ijcai25}

\usepackage{times}
\usepackage{soul}
\usepackage{url}
\usepackage[hidelinks]{hyperref}
\usepackage[utf8]{inputenc}
\usepackage[small]{caption}
\usepackage{graphicx}
\usepackage{amsmath,amssymb,amsfonts,mathrsfs}
\usepackage{amsthm}
\usepackage{booktabs}
\usepackage{algorithm}
\usepackage{algorithmic}
\usepackage{bm}
\usepackage{subfig}
\usepackage{multirow}
\usepackage{float}
\usepackage[switch]{lineno}
\usepackage{wrapfig}

\allowdisplaybreaks[4]

\newtheorem{theorem}{Theorem}

\newtheorem{proposition}[theorem]{Proposition}

\newtheorem{definition}[theorem]{Definition}

\usepackage{enumitem}
\setlist{nolistsep}

\DeclareMathOperator*{\argmin}{arg\,min}



\title{Linear Trading Position with Sparse Spectrum}

\author{
Zhao-Rong Lai$^1$
\and
Haisheng Yang$^{2,}$\thanks{Corresponding author.}\\
\affiliations
$^1$Guangdong Key Laboratory of Data Security and Privacy Preserving, \\  College of Cyber Security, Jinan University\\
$^2$Lingnan College, Sun Yat-Sen University\\
\emails
laizhr@jnu.edu.cn, 
yhaish@mail.sysu.edu.cn
}

\begin{document}

\def \bbE {\mathbb E}
\def \bbR {\mathbb R}
\def \bbN {\mathbb N}
\def \bbZ {\mathbb Z}

\def \mA {\mathcal{A}}
\def \mD {\mathcal{D}}
\def \mF {\mathcal{F}}
\def \mG {\mathcal{G}}
\def \mI {\mathcal{I}}
\def \mJ {\mathcal{J}}
\def \mL {\mathcal{L}}
\def \mM {\mathcal{M}}
\def \mN {\mathcal{N}}
\def \mO {\mathcal{O}}
\def \mP {\mathcal{P}}
\def \mQ {\mathcal{Q}}
\def \mS {\mathcal{S}}
\def \mT {\mathcal{T}}
\def \mU {\mathcal{U}}
\def \mZ {\mathcal{Z}}

\def \sF {\mathscr{F}}
\def \sG {\mathscr{G}}
\def \sL {\mathscr{L}}
\def \sS {\mathscr{S}}
\def \sT {\mathscr{T}}

\def \ba {\bm{a}}
\def \bb {\bm{b}}
\def \bc {\bm{c}}
\def \bd {\bm{d}}
\def \bh {\bm{h}}
\def \bp {\bm{p}}
\def \bq {\bm{q}}
\def \bx {\bm{x}}
\def \by {\bm{y}}
\def \bz {\bm{z}}
\def \bw {\bm{w}}
\def \bu {\bm{u}}
\def \bv {\bm{v}}
\def \br {\bm{r}}
\def \bs {\bm{s}}
\def \bR {\bm{R}}
\def \bS {\bm{S}}
\def \bI {\bm{I}}
\def \bA {\bm{A}}
\def \bB {\bm{B}}
\def \bC {\bm{C}}
\def \bD {\bm{D}}
\def \bE {\bm{E}}
\def \bF {\bm{F}}
\def \bG {\bm{G}}
\def \bH {\bm{H}}
\def \bL {\bm{L}}
\def \bP {\bm{P}}
\def \bQ {\bm{Q}}
\def \bR {\bm{R}}
\def \bS {\bm{S}}
\def \bU {\bm{U}}
\def \bV {\bm{V}}
\def \bW {\bm{W}}
\def \bX {\bm{X}}
\def \bY {\bm{Y}}

\def \bPi {\bm{\Pi}}

\def \bone {\mathbf{1}}
\def \bzer {\mathbf{0}}

\def \blambda {\bm{\lambda}}
\def \bLambda {\bm{\Lambda}}
\def \bSigma {\bm{\Sigma}}
\def \bGamma {\bm{\Gamma}}

\def \ff {\mathfrak{f}}
\def \fl {\mathfrak{l}}
\def \fp {\mathfrak{p}}
\def \frs {\mathfrak{s}}

\def \bbRn {\bbR^{n}}
\def \bbRm {\bbR^{m}}
\def \bbRN {\bbR^{N}}
\def \bbRNN {\bbR^{N\times N}}
\def \bbRM {\bbR^{M}}
\def \bbRT {\bbR^{T}}
\def \bbNM {\bbN_M}
\def \bbNk {\bbN_k}
\def \bbNn {\bbN_n}
\def \bbNm {\bbN_m}

\def \tS {\tilde{S}}
\def \tbQ {\tilde{\bQ}}
\def \tbq {\widetilde{\bq}}
\def \tbI {\widetilde{\bI}}

\def \st {\text{s.\ t.}}
\def \tr {\mathrm{tr}}
\def \prox {\mathrm{prox}}
\def \proj {\mathrm{proj}}
\def \sign {\mathrm{sign}}
\def \env {\mathrm{env}}
\def \supp {\mathrm{supp}}
\def \trun {\mathrm{trun}}
\def \dist {\mathrm{dist}}
\def \diag {\mathrm{diag}}
\def \Fix {\mathrm{Fix}}
\def \VaR {\mathrm{VaR}}
\def \CVaR {\mathrm{CVaR}}
\def \IR {\mathrm{IR}}
\def \MR {\mathrm{MR}}
\def \SR {\mathrm{SR}}
\def \gra {\mathrm{gra}\hspace{2pt}}
\def \dom {\mathrm{dom}\hspace{2pt}}
\def \crit {\mathrm{crit}\hspace{2pt}}

\newcommand\leqs{\leqslant}
\newcommand\geqs{\geqslant}
\newcommand{\ud}{\,\mathrm{d}}

\maketitle

\begin{abstract}
The principal portfolio approach is an emerging method in signal-based trading. However, these principal portfolios may not be diversified to explore the key features of the prediction matrix or robust to different situations. To address this problem, we propose a novel linear trading position with sparse spectrum that can explore a larger spectral region of the prediction matrix. We also develop a Krasnosel'ski\u \i-Mann fixed-point algorithm to optimize this trading position, which possesses the descent property and achieves a linear convergence rate in the objective value. This is a new theoretical result for this type of algorithms. Extensive experiments show that the proposed method achieves good and robust performance in various situations.
\end{abstract}

\section{Introduction}
\renewcommand{\thefootnote}{\fnsymbol{footnote}}
\footnotetext[7]{The supplementary material and code for this paper are available at \url{https://github.com/laizhr/LTPSS}.}
\renewcommand{\thefootnote}{\arabic{footnote}}
In general, asset pricing starts with finding some current signals $\bS_t\in \bbRN$ as proxies for the conditional expected returns $\bR_{t+1}\in \bbRN$ at the next time, which gives rise to signal-based trading \cite{sigtrade}. Since asset correlation exists in most cases, linear trading strategies are proposed to exploit cross-predictability (see Figure \ref{fig:lineartrade}) \cite{lintrade1,lintrade2,desingijcai,AMSPA,MSSRM}. Recently, \cite{PPA} propose to directly combine \textbf{signals} and \textbf{returns} into a \textbf{bi-linear} form with a ``prediction matrix'', and extract several principal portfolios (PP) of this prediction matrix to form a linear trading position (LTP).  \cite{PPA2} further propose a conditional factor model for individual corporate bond based on instrumented principal components analysis (IPCA,\cite{IPCA}). This approach exploits the predictability from both own-asset signal and other-asset signals, while keeping a closed-form LTP. It not only keeps a concise portfolio but also improves interpretability of LTP in finance, which leads a new direction in future research.

\begin{figure}[!htb]
			\centering
\includegraphics[width=0.5\columnwidth]{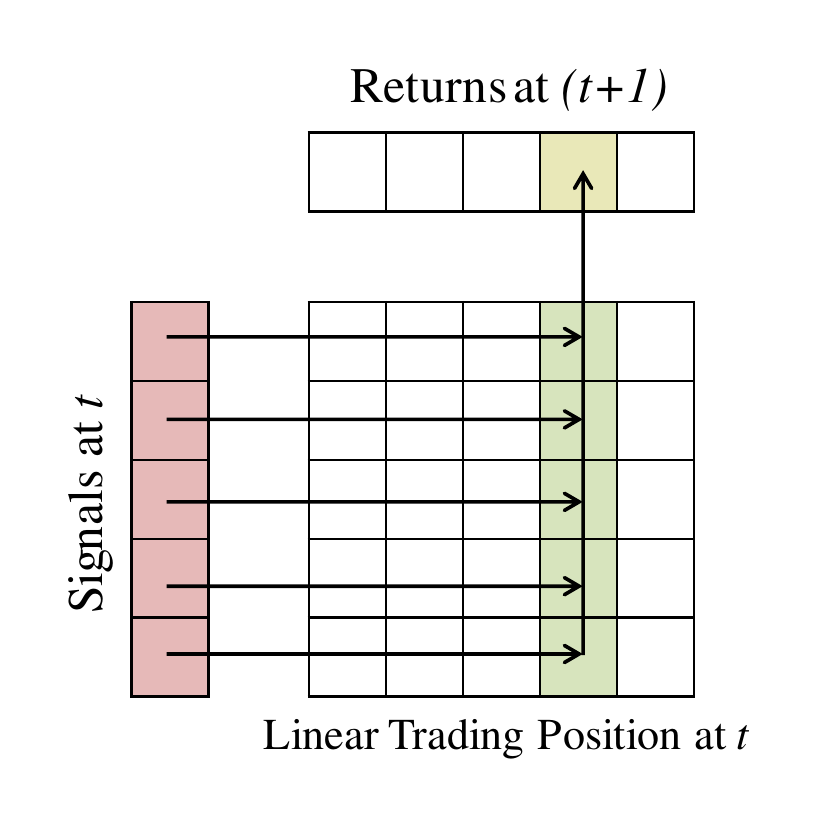}\\
\includegraphics[width=0.9\columnwidth]{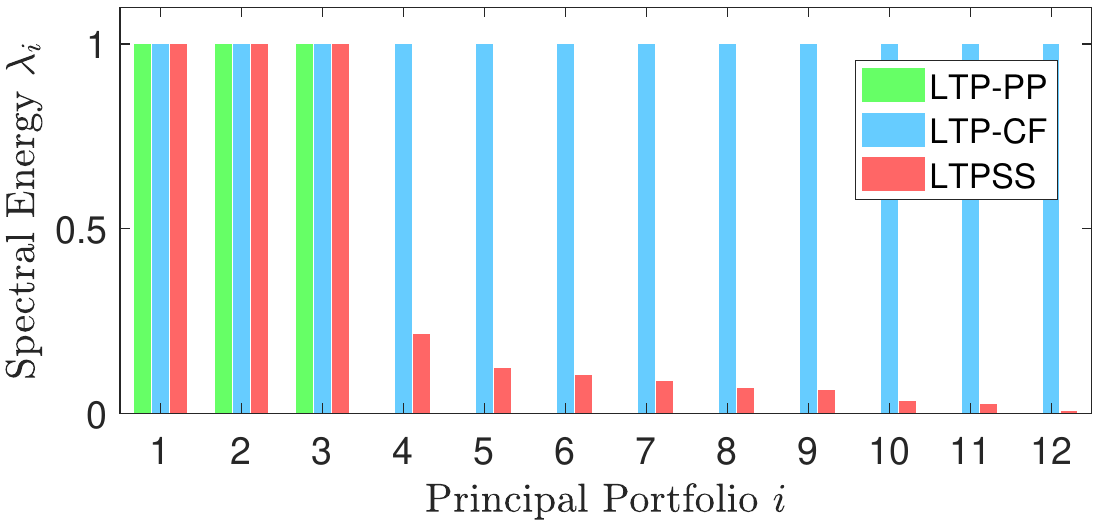}
		\caption{\textbf{Top:} cross-sectional linear trading position. The arrows show the information flow from the signals of all the assets to the position of the fourth asset. \textbf{Bottom:} the closed-form solution (blue bars, LTP-CF) employs all the principal portfolios; the principal portfolio strategy (green bars, LTP-PP) selects the first several principal portfolios; while the proposed method (red bars, LTPSS) sets diversified spectral energies in $[0,1]$.}
		\label{fig:lineartrade}
	\end{figure}

However, LTP composed by several PPs may not be diversified to explore the key features of the prediction matrix. First, it is recognized in \cite{PPA} that including more PPs into the LTP does not necessarily lead to better performance. Second, the spectral energy of each PP is either $1$ (selected) or $0$ (non-selected), and cannot be adjusted in the current framework (green or blue bars in Figure \ref{fig:lineartrade}). Third, the choice of PPs relies on empirical experience. \cite{PPA} suggest empirically using the first three PPs (green bars in Figure \ref{fig:lineartrade}) as they possess the most proportion of singular value magnitude, but it may not be robust to different situations (see Supplementary \ref{sup:unstablepp}). 

\textbf{Importance of diversification:} given a simple economic system with $10$ states and $10$ assets. On each trading period, only one state can occur, and each state occurs with $10\%$ probability. If state $i$ occurs, the price of Asset $i$ will increase by $10\%$, while those of others will decrease by $1\%$. Then if an investor invests all his/her wealth in only one asset, he/she will have only $10\%$ probability to gain $10\%$ return, but $90\%$ probability to lose $1\%$. However, if he/she diversifies the wealth equally in $10$ assets (each with $10\%$ of the position), he/she will have $100\%$ probability to gain $1\%$ return without taking the risk of losing wealth. 

The same law also holds for PPs, where each PP can be considered as an asset in the above example. Then the \textbf{diversification over PPs leads to the sparse spectrum concept} that has been widely used in various machine learning scenarios, such as matrix completion \cite{singshrink,matcom}, knowledge base completion \cite{nunormknow}, graph representation learning \cite{nunormgraph}, and insufficient-label recognition \cite{nunormcv}. We aim to employ this concept in the LTP framework to extract the key features of the prediction matrix (red bars in Figure \ref{fig:lineartrade}). There are some \textbf{main difficulties: 1.} This problem has a three-part complicated non-differentiable geometrical structure with the Frobenius, nuclear, and spectral norms. A common solver is the semi-definite programming (SDP) with a conic reformulation, which is a surrogate model that cannot achieve full-reinvestment and thus result in poor performance. \textbf{2.} The projection onto the self-financing constraint \cite{SPOLC,egrmvgap} is nonlinear and non-orthogonal, which rules out most projected subgradient approaches that require orthogonality and the inner product property of the Hilbert space. \textbf{3.} The descent property and the convergence rate of the solving algorithm are difficult to established.

To address the above challenges, we mainly offer the following contributions. \textbf{1.} We propose a Linear Trading Position with Sparse Spectrum (LTPSS) that can explore a larger spectral region of the observation matrix, while keeping a sparse and concise representation. \textbf{2.} We prove that the nonlinear and non-orthogonal projection onto the feasible set is non-expansive, which is crucial to the convergence of the whole solving algorithm. \textbf{3.} We develop a Krasnosel'ski\u \i-Mann (KM) fixed-point algorithm to solve LTPSS. It possesses the descent property and achieves a linear convergence rate in the objective value. To the best of our knowledge, this is a new theoretical result for KM algorithms, as the current best result is ${ o} \left(\frac{1}{{ k}} \right)$ in the fixed point iteration gap \cite{KMconverrate}, but not in the objective value. This new finding may reveal greater breakthroughs for KM algorithms.

\section{Preliminaries and Related Works}
\label{sec:relate}

\subsection{Linear Trading Position with Principal Portfolios}
\label{sec:subsecLTPPP}

We start with introducing the concepts of LTP and PP. Let $\bS_t\in \bbRN$ and $\bR_{t+1}\in \bbRN$ be the signals at time $t$ and returns at time $(t+1)$ for the $N$ assets in a financial market, respectively. The portfolio optimization task is to find an LTP $\bL\in\bbRNN$ such that
\begin{equation}
\label{eqn:ltporg}
\max \bbE[\bS_t^\top \bL \bR_{t+1}].
\end{equation}
In finance, a self-financing constraint is usually imposed to ensure a realistic and feasible trading position, which means that no external money can be added to the position once the trading strategy starts \cite{SPOLC,egrmvgap}. In the framework of \cite{PPA,PPA2}, the following self-financing constraint is used
\begin{equation}
\label{eqn:consl2}
\Omega:=\{\bL\in \bbRNN: \|\bL\|_{2}\leqs 1\},
\end{equation}
where $\|\cdot\|_{2}$ denotes the $\ell_2$ norm (spectral norm) of a matrix. A trivial strategy satisfying this constraint is the simple factor:
\begin{align}
\label{eqn:simfact}
\hat{\bL}_{SF}:=\bI_{(N)}, \tag{LTP-SF}
\end{align}
where $\bI_{(N)}$ denotes the identity matrix of $N$ dimensions. It only uses the own signal of each asset to determine its position. The original LTP model can be formulated as:
\begin{align}
\label{eqn:ltporg2}
\max_{\bL\in\Omega} \tr(\bL\bPi),\quad \bPi:=\bbE[\bR_{t+1}\bS_t^\top], \tag{LTP}
\end{align}
where $\tr(\cdot)$ denotes the trace operator of a matrix. This formulation exploits the commutative law and the linearity of the trace operator: 
\[
\bbE[\bS_t^\top \bL \bR_{t+1}]=\bbE[ \tr(\bL \bR_{t+1}\bS_t^\top)]= \tr(\bL \bbE[\bR_{t+1}\bS_t^\top]).
\]
$\bPi$ is called a \emph{prediction matrix}, which contains all the information that can be exploited to determine $\bL$. \cite{PPA} propose to estimate it by the empirical estimator:
\begin{align}
\label{eqn:piestimate}
\hat{\bPi}:=\frac{1}{T} \sum_{\tau=t-T}^{t-1}\bR_{\tau+1}\bS_\tau^\top.
\end{align}
Since the matrix $\ell_2$ norm is equivalent to the Schatten $\infty$-norm, it follows from the H\"older's inequality that
\begin{align}
\label{eqn:ltpholder}
\tr(\bL\hat{\bPi})\leqs \|\bL\|_2 \|\hat{\bPi}\|_*,
\end{align}
where $\|\cdot\|_*$ denotes the nuclear norm of the matrix. The equality holds if and only if $\bL$ is linear to $(\hat{\bPi}^\top\hat{\bPi})^{-\frac{1}{2}}\hat{\bPi}^\top$. Since $\|\bL\|_2\leqs 1$, 
\begin{align}
\label{eqn:ltpcf}
\hat{\bL}_{CF}:=(\hat{\bPi}^\top\hat{\bPi})^{-\frac{1}{2}}\hat{\bPi}^\top \tag{LTP-CF}
\end{align}
is exactly the closed-form solution to \eqref{eqn:ltporg2} with the empirical estimator $\hat{\bPi}$.

Instead of $\hat{\bL}_{CF}$, \cite{PPA} propose to use the PPs of $\hat{\bPi}$ to develop the LTP $\bL$. To do this, the first step is to conduct the singular value decomposition (SVD) of $\hat{\bPi}^\top$ as
\begin{align}
\label{eqn:svdpi}
\hat{\bPi}^\top{:=}\bU {\bSigma}\bV^\top, \bSigma{:=}\diag(\{{\sigma}_i\}_{i=1}^N), {\sigma}_1{\geqs} {\sigma}_2{\geqs} \cdots{\geqs} {\sigma}_N{\geqs} 0,
\end{align}
where $\bU,\bV,\bU^\top,\bV^\top$ are orthonormal bases of $\bbRNN$. Then the principal components $\{\bu_n \bv_n^\top\}_{n=1}^N$ are defined as the PPs of the prediction matrix $\hat{\bPi}$, where $\bu_n$ and $\bv_n$ are the $n$-th columns of the matrices $\bU$ and $\bV$, respectively. \cite{PPA} propose to use the first several PPs to construct an LTP:
\begin{align}
\label{eqn:ltppp}
\hat{\bL}_{PP}:=\sum_{n=1}^l \bu_n \bv_n^\top, \tag{LTP-PP}
\end{align}
where $l$ denotes the number of selected PPs. This LTP has a simple form, and is highly interpretable in the literature of finance. However, it may not be diversified to capture the key features of $\hat{\bPi}$. First, the setting of $l$ relies on empirical experience. As indicated by \cite{PPA}, a larger $l$ does not necessarily lead to better results. Second, \eqref{eqn:ltppp} indicates that the spectral energy of a selected PP is fixed as $1$, which may not be adaptive to different situations.

\subsection{Surrogate Model}
\label{sec:surrogate}
There are some surrogate models for the proposed \eqref{eqn:LTPSS} model in Section \ref{sec:LTPSSmodel}. One main approach is to reformulate \eqref{eqn:LTPSS} as a semi-definite conic programming\footnote{\url{https://www.seas.ucla.edu/~vandenbe/236C/lectures/conic.pdf}} (SDCP). Specifically,
\begin{align}
\label{eqn:ltpsssurro}
&\|\bL\|_{2}\leqs 1 \Leftrightarrow \begin{bmatrix}  
\bI_{(N)} & \bL  \\
 \bL^\top   &  \bI_{(N)}
\end{bmatrix}\succeq  0,\\
\label{eqn:ltpsssurro2}
&\|\bL\|_{*}\leqs \fl \Leftrightarrow \begin{bmatrix}  
\bU & \bL  \\
 \bL^\top   &  \bV
\end{bmatrix}\succeq  0, \ \text{s.t.}\ \frac{1}{2} (\tr{\bU}+\tr{\bV})\leqs \fl,
\end{align}
where $\bU$, $\bV$, and $\fl$ are auxiliary arguments that need to be optimized simultaneously with $\bL$. Then the surrogate model of \eqref{eqn:LTPSS} is:
\begin{align}
\label{eqn:LTPSSsurro}
& \min_{\bL, \bU, \bV, \fl}   -\tr(\bL\hat{\bPi})+\eta\fl,  \nonumber\\
\text{s.t.}& \begin{bmatrix}  
\bI_{(N)} & \bL  \\
 \bL^\top   &  \bI_{(N)}
\end{bmatrix}{\succeq}  0,\ \begin{bmatrix}  
\bU & \bL  \\
 \bL^\top   &  \bV
\end{bmatrix}{\succeq}  0, \ \frac{1}{2} (\tr{\bU}{+}\tr{\bV})\leqs \fl. \tag{SDCP}
\end{align}
\textbf{This model is not equivalent to \eqref{eqn:LTPSS}.} It directly replaces $\|\bL\|_{*}$ in \eqref{eqn:LTPSS} by $\fl$, which is only an upper bound of $\|\bL\|_{*}$. Besides, the mainstream solvers for \eqref{eqn:LTPSSsurro} are based on interior-point primal-dual algorithms, which \textbf{cannot achieve the equality} $\|L\|_2 = 1$ but only satisfy $\|\bL\|_{2}< 1$. In the experiments, we observe that $\|L\|_2\approx 0.9826$ in all the cases, with both absolute solution tolerance and constraint tolerance being $1e-6$. In this case, \textbf{the investing capital can not be fully exploited} and the investing performance is unsatisfactory (see Table \ref{tab:MR}). \textbf{Worse still, it significantly increases computational complexity due to the auxiliary arguments $\bU$, $\bV$, and $\fl$}. The dimensionality for the constraints of \eqref{eqn:LTPSSsurro} quadruples that of \eqref{eqn:LTPSS}.

\section{Methodology}
\label{sec:LTPSS}
To address the above problems, we propose to expand the spectral region of LTP. First, we consider all the PPs to construct the LTP. Second, we allow for flexible spectral energies in $[0,1]$. Third, we impose sparse spectrum on the LTP to extract the key features of the prediction matrix $\hat{\bPi}$. Fourth, we develop a KM fixed-point algorithm to directly solve \eqref{eqn:LTPSS} instead of using the defective and deficient surrogate model \eqref{eqn:LTPSSsurro}.

\subsection{Linear Trading Position with Sparse Spectrum}
\label{sec:LTPSSmodel}
Since an LTP $\bL\in \bbRNN$, our framework is developed on the linear space $\bbRNN$, which is more complicated than $\bbRN$. First of all, we consider $\bbRNN$ as a Hilbert space, then the trace operator $\tr(\bA^\top\bB)=:\langle\bA,\bB\rangle$ is the inner product of $\bA$ and $\bB$, for any $\bA,\bB\in\bbRNN$. Moreover, $\sqrt{\tr(\bA^\top\bA)}=\|\bA\|_F$ and thus the Frobenius norm is the induced norm from the trace operator for $\bbRNN$. First, we give some properties of the constraint set $\Omega$ defined in \eqref{eqn:consl2}.
\begin{proposition}
\label{prop:conclobou}
$\Omega$ is a convex, closed, and bounded subset of $\bbRNN$. 
\end{proposition}

\begin{proof}
$\forall \bA,\bB\in \Omega$ and $\forall \theta \in [0,1]$,
\begin{align}
\label{eqn:consl2conv}
\|\theta\bA{+}(1{-}\theta)\bB\|_{2}{\leqs}\theta\|\bA\|_{2}{+}(1{-}\theta)\|\bB\|_{2}{\leqs} \theta{+}(1{-}\theta){=}1.
\end{align}
Hence $\|\theta\bA+(1-\theta)\bB\|_{2}\in \Omega$ and $\Omega$ is convex. Since $\Omega$ is a sub-level-set of the continuous function $\|\cdot\|_2$, it is closed \cite{varana}. $\|\bL\|_F=\sqrt{\sum_{i=1}^N \lambda_i^2}$, where $\{\lambda_i\}_{i=1}^N$ are the singular values of $\bL$ (singular values are non-negative). \eqref{eqn:consl2} indicates that $(\max_{1\leqs i\leqs N} \lambda_i)\leqs 1$. Hence $\|\bL\|_F\leqs \sqrt{N}$ for any $\bL\in \Omega$, which proves that $\Omega$ is bounded.
\end{proof}

Next, we propose the following LTPSS model:
\begin{align}
\label{eqn:LTPSS}
\min_{\bL\in\Omega} F(\bL):=f(\bL)+g(\bL):=-\tr(\bL\hat{\bPi})+\eta\|\bL\|_*, \tag{LTPSS}
\end{align}
where $\eta\geqs 0$ is a regularization parameter. $f(\bL)$ has the gradient $\nabla f(\bL)=-\hat{\bPi}^\top$, while $g(\bL)$ is non-differentiable. Furthermore, the constraint $\bL\in \Omega$ is also a non-differentiable structure. In summary, \textbf{\eqref{eqn:LTPSS} has three parts with the Frobenius, nuclear, and spectral norms, where the latter two are non-differentiable. By dropping either $g(\bL)$ or $\bL\in \Omega$, it can be reduced to a common problem that can be solved by either projected gradient or proximal gradient methods, respectively. However, with both non-differentiable parts present, \eqref{eqn:LTPSS} cannot be effectively and efficiently solved via common approaches.} 

\textbf{Fact:} there exists at least one solution to \eqref{eqn:LTPSS}. Since $F(\bL)$ is a continuous function on $\bL$ and $\Omega$ is a convex and compact set, this fact follows from the Weierstrass extreme value theorem.

\subsection{Krasnosel'ski\u \i-Mann Fixed-point Algorithm}
\label{sec:PPGA}
To address the above difficulty, we develop a KM fixed-point scheme to solve \eqref{eqn:LTPSS}, which alternately minimizes $F(\bL)$ and projects $\bL$ onto $\Omega$. Denote the $k$-th iterate of $\bL$ by $\bL^{(k)}$. Then in the next iteration, we can use a quadratic approximation to $F(\bL)$:
\begin{align}
\label{eqn:quadapprox2}
&Q(\bL,\bL^{(k)})=-\tr(\bL^{(k)}\hat{\bPi})-\frac{\beta}{2}\|\hat{\bPi}\|_F^2\nonumber\\
&\qquad\quad+\frac{1}{2\beta}\|\bL-\bL^{(k)}-\beta\hat{\bPi}^\top\|_F^2+\eta\|\bL\|_*,
\end{align}
where $\beta>0$ is the step size parameter (which can be seen later). By ignoring the constant terms with respect to (w.r.t.) the argument $\bL$ and temporarily relaxing the constraint $\bL\in\Omega$, we minimize $Q(\bL,\bL^{(k)})$:
\begin{align}
\label{eqn:Qfuncprox}
&\min_{\bL\in\bbRNN} Q(\bL,\bL^{(k)})\quad\Longleftrightarrow\nonumber\\
&\argmin_{\bL\in\bbRNN} \left\{  \frac{1}{2\beta}\|\bL-(\bL^{(k)}+\beta\hat{\bPi}^\top)\|_F^2+\eta\|\bL\|_*  \right \}\nonumber\\
&\quad =: \prox_{\beta\eta\|\cdot\|_*}(\bL^{(k)}+\beta\hat{\bPi}^\top),
\end{align}
which is the proximal mapping of $(\bL^{(k)}+\beta\hat{\bPi}^\top)$ w.r.t. the function $\beta\eta\|\cdot\|_*$. It has a closed form solution \cite{singshrink}. Conduct the SVD of $(\bL^{(k)}+\beta\hat{\bPi}^\top)$ as
\begin{align}
\label{eqn:proxsvd}
\bL^{(k)}+\beta\hat{\bPi}^\top=\tilde{\bU} \bLambda\tilde{\bV}^\top,\ \bLambda:=\diag(\{\lambda_i\}_{i=1}^N).
\end{align}
Then the singular value thresholding of $\bLambda$ is 
\begin{align}
\label{eqn:proxsvt}
\tilde{\bLambda}:=\diag(\{\sign(\lambda_i)\cdot\max\{|\lambda_i|-\beta\eta,0\}\}_{i=1}^N).
\end{align}
To be intuitive, $\tilde{\bLambda}$ drags each singular value of $\bLambda$ towards $0$ by a step $\beta\eta$. The closed form solution to \eqref{eqn:Qfuncprox} is
\begin{align}
\label{eqn:Qfuncproxsol}
\sG(\bL)&{:=}\bL{+}\beta\hat{\bPi}^\top, \tilde{\bL}^{(k)}{:=}\prox_{\beta\eta\|\cdot\|_*}(\sG(\bL^{(k)})){=} \tilde{\bU} \tilde{\bLambda}\tilde{\bV}^\top,
\end{align}
which is actually a proximal-gradient operator. Next, we need to design a feasible projection operator to project $\tilde{\bL}^{(k)}$ onto the constraint set $\Omega$.
\begin{definition}
\label{dfn:projomega}
Recall the singular vectors of $\hat{\bPi}^\top$ defined in \eqref{eqn:svdpi} as $\bU$ and $\bV$ . Given any matrix $\bA\in\bbRNN$, let ${\bGamma}:=\bU^\top \bA \bV$ and ${\gamma}_{ii}$ be the $i$-th diagonal element of ${\bGamma}$. Define
\begin{align}
\label{eqn:projsvd}
&\check{\bGamma}:=\diag(\{\check{\gamma}_i\}_{i=1}^N), \ \check{\gamma}_i:=\begin{cases}
\sign{({\gamma}_{ii})}, &\text{if}\ |{\gamma}_{ii}|>1  ;\\
{\gamma}_{ii}, &\text{if}\  |{\gamma}_{ii}|\leqs 1,
\end{cases},\nonumber\\
&\proj_{\Omega} (\bA):=\bU \check{\bGamma}\bV^\top.
\end{align}
\end{definition}
\textbf{Note that $\proj_{\Omega}$ is a nonlinear and non-orthogonal projection, but it is a non-expansive operator, which is crucial to the convergence of the whole algorithm.}

\begin{theorem}
\label{thm:proj}
$\proj_{\Omega}$ defined in \eqref{eqn:projsvd} is non-expansive.
\end{theorem}

\begin{proof}
First, it can be easily observed from the definition of $\proj_{\Omega}$ that $\proj_{\Omega}\circ\proj_{\Omega}=\proj_{\Omega}$: the first $\proj_{\Omega}$ makes every $\check{\gamma}_i$ lie in $[-1,1]$, then the second $\proj_{\Omega}$ will not change $\check{\gamma}_i$ any more. Thus $\proj_{\Omega}$ is a projection (not necessarily linear or orthogonal) by definition. Second, to prove $\proj_{\Omega}$ is non-expansive, we need to verify that 
\begin{align}
\label{eqn:nonexpand}
\|\proj_{\Omega} (\bA){-}\proj_{\Omega} (\bB)\|_F{\leqs} \| \bA{-}\bB\|_F, \ \forall \bA,\bB\in \bbRNN.
\end{align}
Let $\bA=\bU(\bU^\top \bA \bV)\bV^\top:=\bU \bGamma\bV^\top$ and $\bB=\bU(\bU^\top \bB \bV)\bV^\top:=\bU \bLambda\bV^\top$. Then $\proj_{\Omega} (\bA)=\bU \check{\bGamma}\bV^\top$ and $\proj_{\Omega} (\bB)=\bU \check{\bLambda}\bV^\top$. \eqref{eqn:nonexpand} is equivalent to
\begin{align}
\label{eqn:nonexpand2}
&\|\bU \check{\bGamma}\bV^\top-\bU \check{\bLambda}\bV^\top\|_F^2   {\leqs} \| \bU\bGamma\bV^\top{-}\bU \bLambda\bV^\top\|_F^2,\nonumber\\
&\tr[(\bU \check{\bGamma}\bV^\top{-}\bU \check{\bLambda}\bV^\top)^\top(\bU \check{\bGamma}\bV^\top{-}\bU \check{\bLambda}\bV^\top)]\nonumber\\
{\leqs}&\tr[(\bU {\bGamma}\bV^\top{-}\bU {\bLambda}\bV^\top)^\top(\bU {\bGamma}\bV^\top{-}\bU {\bLambda}\bV^\top)].\qquad\quad
\end{align}
Since $\bU$ and $\bV$ are orthonormal bases of $\bbRN$,
\begin{align}
\label{eqn:trquad}
&\|\bU {\bGamma}\bV^\top\|_F^2=  \tr(\bV{\bGamma}^\top\bU^\top\bU {\bGamma}\bV^\top)=\tr(\bV{\bGamma}^\top\bGamma\bV^\top)\nonumber\\
=&\tr(\bV^\top\bV{\bGamma}^\top\bGamma)=\tr({\bGamma}^\top\bGamma)=\|\bGamma\|_F^2.\qquad\quad
\end{align}
By similar transformations, \eqref{eqn:nonexpand2} can be simplified as
\begin{align}
\label{eqn:nonexpand3}
&\|\check{\bGamma}\|_F^2{+}\|\check{\bLambda}\|_F^2{-}2\tr(\check{\bLambda} \check{\bGamma}){\leqs}\|{\bGamma}\|_F^2{+}\|{\bLambda}\|_F^2{-}2\tr({\bLambda}^\top {\bGamma})\\
\label{eqn:nonexpand6}
&\Leftrightarrow \ \sum_{i=1}^N (\check{\gamma_{i}}-\check{\lambda_{i}})^2 \leqs\sum_{i=1}^N \sum_{j=1}^N (\gamma_{ij}-\lambda_{ij})^2.
\end{align}
\eqref{eqn:nonexpand6} will hold if
\begin{align}
\label{eqn:nonexpand6csineq}
|\check{\gamma_{i}}-\check{\lambda_{i}}|\leqs |\gamma_{ii}-\lambda_{ii}|,\quad i=1,\cdots,N.
\end{align}
\eqref{eqn:projsvd} indicates that $\check{\gamma_{i}}$ and $\check{\lambda_{i}}$ are one-dimensional projections of $\gamma_{ii}$ and $\lambda_{ii}$ onto the compact interval $[-1,1]$, respectively. Such one-dimensional projection is non-expansive, which means that \eqref{eqn:nonexpand6csineq} holds. Back from  \eqref{eqn:nonexpand6csineq} to \eqref{eqn:nonexpand}, $\proj_{\Omega}$ is non-expansive.
\end{proof} 

Now $\tilde{\bL}^k$ can be projected onto $\Omega$ by
\begin{equation}
\setlength{\abovedisplayskip}{3pt}
\setlength{\belowdisplayskip}{3pt}
\label{eqn:projsvd2}
\proj_{\Omega} (\tilde{\bL}^{(k)})=\bU \hat{\bLambda}\bV^\top.
\end{equation}
\eqref{eqn:Qfuncproxsol} and \eqref{eqn:projsvd2} yield a composed operator:
\begin{align}
\label{eqn:iteration}
\sT(\bL^{(k)}):=\proj_{\Omega}\circ \prox_{\beta\eta\|\cdot\|_*}\circ\sG(\bL^{(k)}).
\end{align}

\begin{theorem}
\label{thm:nonexpan}
$\sT: \bbRNN \rightarrow \bbRNN$ is a non-expansive operator.
\end{theorem}
The proof is put in Supplementary \ref{proof:nonexpan}.

\begin{proposition}
\label{prop:fixedpointexist}
There exists a fixed-point in $\Omega$ for $\sT$.
\end{proposition}
The proof is put in Supplementary \ref{proof:fixedpointexist}. Denote $\mF_\Omega:=\{\bA\in \Omega: \sT(\bA)=\bA\}\ne \emptyset$ as the fixed point set of $\sT$ on $\Omega$. The next step is to develop a convergent algorithm with $\sT$. Given any initial point $\bL^{(0)}$, let $\bL^{(1)}:=\sT(\bL^{(0)})$ and then $\bL^{(1)}\in \Omega$. We develop a KM iteration with $\theta\in (0,1)$ as follows
\begin{align}
\label{eqn:KMiter}
\bL^{(k+1)}:=(1-\theta)\bL^{(k)}+\theta\sT(\bL^{(k)}),\quad k=1,2,\cdots
\end{align}
Since $\bL^{(1)}\in \Omega$, $\sT(\bL^{(1)})\in \Omega$, and $\Omega$ is convex, \eqref{eqn:KMiter} implies $\bL^{(2)}\in \Omega$. It follows from mathematical induction that the whole iterative sequence $\{\bL^{(k)}\}_{k\geqs 1}\subseteq \Omega$.

\subsection{Descent Property and Linear Convergence Rate}
\label{sec:despropconrate}
\begin{theorem}[Krasnosel'ski\u \i-Mann]
\label{thm:KMtheorem}
The iterative sequence $\{\bL^{(k)}\}_{k\geqs 1}$ produced by \eqref{eqn:KMiter} converges to a fixed point $\bL^*\in\mF_\Omega$.
\end{theorem}

The proof of this theorem is put in Supplementary \ref{proof:KMtheorem}. In addition, the corresponding objective value $F(\bL^{(k)})$ also descends and converges with a linear convergence rate. This is a new theoretical result for KM algorithms, as the current best result is ${ o} \left(\frac{1}{{ k}} \right)$ in the fixed point iteration gap $\|(\bL^{(k)}-\sT(\bL^{(k)})\|_F$ \cite{KMconverrate}, but not in the objective value. In fact, it is difficult to deduce the convergence rate of KM algorithms in the objective value for general problems. 
\begin{theorem}
\label{thm:objdescent}
The iterative sequence $\{\bL^{(k)}\}_{k\geqs 1}$ produced by \eqref{eqn:KMiter} satisfies the descent property:
\begin{align}
\label{eqn:objdescent}
&F(\bL^{(k+1)})-F(\bL^{(k)})\leqs  0,\quad \forall k\geqs 1.
\end{align}
Moreover, 

(a) If $\sigma_i \ne \eta$ and ${\lambda}^{(k)}_i\ne 0$ for some $i$, then $F(\bL^{(k+1)})-F(\bL^{(k)})< 0$.

(b) If $\sigma_i \ne \eta$ for all $i$, then $F(\bL^{(k)})$ achieves a linear convergence rate.
\end{theorem}

\begin{proof}
\textbf{Part (a):} Given any initial $\bL^{(0)}$, since the last component of $\sT$ is $\proj_{\Omega}$ defined in \eqref{eqn:projsvd}, $\bL^{(1)}$ has the following form:
\begin{align*}
\bL^{(1)}&{=}\sT(\bL^{(0)}){=}\proj_{\Omega}{\circ} \prox_{\beta\eta\|\cdot\|_*}{\circ}\sG(\bL^{(0)}){=}\bU {\bLambda}^{(1)}\bV^\top,
\end{align*}
where $\bLambda^{(1)}:=\diag(\{{\lambda}^{(1)}_i\}_{i=1}^N)$ with ${\lambda}^{(1)}_i\in [-1, 1]$ for all $i$. Suppose 
\begin{align}
\label{eqn:svdmathinduce}
\bL^{(k)}=\bU {\bLambda}^{(k)}\bV^\top, \quad  {\lambda}^{(k)}_i\in [-1, 1], \ \forall i. 
\end{align}
Similar to the above deduction, 
\begin{align}
\label{eqn:svdmathinduce2}
\sT(\bL^{(k)})=\bU {\bGamma}^{(k)}\bV^\top, \quad  {\gamma}^{(k)}_i\in [-1, 1], \ \forall i. 
\end{align}
Note that $\bL^{(k)}$, $\sT(\bL^{(k)})$ and $\hat{\bPi}^\top$ have the same singular vectors $\bU$ and $\bV$. Hence the operators $\sG$, $\prox_{\beta\eta\|\cdot\|_*}$, and $\proj_{\Omega}$ actually act on the diagonal matrix ${\bLambda}^{(k)}$ in an element-wise way, based on their definitions \eqref{eqn:Qfuncproxsol} and \eqref{eqn:projsvd}. For each $i$,
\begin{align}
\label{eqn:svdmathinduce3}
&{\gamma}^{(k)}_i{=}\begin{cases}\min\{\max\{{\lambda}^{(k)}_i{+}\beta(\sigma_i{-}\eta),0\}, 1\},\ \ \text{if}\ {\lambda}^{(k)}_i{+}\beta\sigma_i{\geqs} 0;\\
\min\{{\lambda}^{(k)}_i+\beta(\sigma_i+\eta),0\}, \quad\text{if}\ {\lambda}^{(k)}_i+\beta\sigma_i<0.
\end{cases}
\end{align}
Since $\beta$, $\sigma_i$, $\eta\geqs 0$ and ${\lambda}^{(k)}_i\geqs -1$, ${\lambda}^{(k)}_i+\beta(\sigma_i+\eta)$ will not be smaller than $-1$. Hence the second case is simplified. It follows from \eqref{eqn:KMiter} that 
\begin{align*}
\bL^{(k+1)}=&\bU[(1-\theta){\bLambda}^{(k)}+\theta {\bGamma}^{(k)}] \bV^\top=:\bU {\bLambda}^{(k+1)}\bV^\top, \\
&{\lambda}^{(k+1)}_i \in[-1, 1], \ \forall i.
\end{align*}
By mathematical induction, \eqref{eqn:svdmathinduce}, \eqref{eqn:svdmathinduce2}, and \eqref{eqn:svdmathinduce3} hold for any $k\geqs 1$. Direct calculation leads to 
\begin{align}
\label{eqn:totaldiff3}
&F(\bL^{(k)})-F(\bL^{(k+1)})\nonumber\\
=&\tr((\bL^{(k+1)}-\bL^{(k)})\hat{\bPi}^\top)+\eta\sum_{i=1}^N(|{\lambda}^{(k)}_i|-|{\lambda}^{(k+1)}_i|)\nonumber\\
=&\sum_{i=1}^N [({\gamma}^{(k)}_i-{\lambda}^{(k)}_i)\theta\sigma_i+\eta(|{\lambda}^{(k)}_i|-|{\lambda}^{(k+1)}_i|)]\nonumber\\
=:& \sum_{i=1}^N (F_i(\bL^{(k)})-F_i(\bL^{(k+1)})),\\
\label{eqn:totaldiff4}
&\text{where} \ \ F_i(\bL^{(k)}):=\eta|{\lambda}^{(k)}_i|-\sigma_i{\lambda}^{(k)}_i.
\end{align}
\textbf{We investigate each $(F_i(\bL^{(k)})-F_i(\bL^{(k+1)}))$ of \eqref{eqn:totaldiff3} in $10$ cases, provided in Supplementary \ref{sup:tencases}}. To summarize, each $(F_i(\bL^{(k)})-F_i(\bL^{(k+1)}))$ of \eqref{eqn:totaldiff3} is non-negative. Hence $F(\bL^{(k)})-F(\bL^{(k+1)})\geqs 0$. Moreover, if $\sigma_i\ne \eta$ and ${\lambda}^{(k)}_i\ne 0$ for some $i$, then $\sigma_i- \eta\ne 0$, $\sigma_i+ \eta\ne 0$, and ${\lambda}^{(k)}_i\ne 0$ for this $i$. It can be seen that all the $10$ cases lead to a positive summand, and $F(\bL^{(k)})-F(\bL^{(k+1)})> 0$.

\textbf{Part (b):} Following similar steps in \eqref{eqn:totaldiff3} and \eqref{eqn:totaldiff4},
\begin{align}
\label{eqn:totaldiff}
&F(\bL^{(1)})-F(\bL^{*})= \sum_{i=1}^N (F_i(\bL^{(1)})-F_i(\bL^{*})).
\end{align}
Part (a) has already verified that each $F_i(\bL^{(1)})$ satisfies the descent property. Moreover, $F_i(\bL^{(1)})\downarrow F_i(\bL^{*})$ for each $i$, otherwise $F(\bL^{(1)})$ cannot monotonically converges to $F(\bL^{*})$. Hence we can break \eqref{eqn:totaldiff} into each $(F_i(\bL^{(1)})-F_i(\bL^{*}))$ for further investigation. 

We further analyze the lower bounds of the $10$ cases in Part (a): $F_i(\bL^{(k)})-F_i(\bL^{(k+1)})= \theta|\eta-\sigma_i||{\lambda}^{(k)}_i|$ in case 6), $F_i(\bL^{(k)})-F_i(\bL^{(k+1)})= \theta|\eta+\sigma_i||{\lambda}^{(k)}_i|$ in cases 7) and 8), and $F_i(\bL^{(k)})-F_i(\bL^{(k+1)})\geqs c>0$ in all the other cases, where $c:=\beta\theta(\sigma_i-\eta)^2$ is a positive constant. 

We start with the simplest situation. If $F_i(\bL^{(k)})$ never falls into cases 6)$\sim$8), then it keeps decreasing by at least a positive constant $c$ at each iteration. It takes at most $k:=\lceil\frac{F_i(\bL^{(1)})-F_i(\bL^{*})}{c} \rceil$ iterations to reach the target objective value. Hence, $F_i(\bL^{(k)})$ converges in constant steps. 

Next, we investigate the situation where $F_i(\bL^{(k)})$ always falls into case 6). Since $F_i(\bL^{(1)})\downarrow F_i(\bL^{*})$ and ${\lambda}^{(k+1)}_i=(1-\theta){\lambda}^{(k)}_i$, we have 
\begin{align}
\label{eqn:dimdiffcase6}
&F_i(\bL^{(1)})-F_i(\bL^{*})=\theta|\eta-\sigma_i|\sum_{m=1}^\infty|{\lambda}^{(m)}_i|\nonumber\\
&\quad=\theta|\eta-\sigma_i||{\lambda}^{(1)}_i|\sum_{m=0}^\infty (1-\theta)^m.
\end{align}
Dropping the first $k$ terms in the sum of \eqref{eqn:dimdiffcase6} yields
\begin{align}
\label{eqn:dimdiffcase6b}
&F_i(\bL^{(k+1)}){-}F_i(\bL^{*})=\theta|\eta{-}\sigma_i||{\lambda}^{(1)}_i|(1{-}\theta)^k\sum_{m=0}^\infty (1{-}\theta)^m\nonumber\\
&=\theta|\eta{-}\sigma_i||{\lambda}^{(1)}_i|(1{-}\theta)^k\frac{1}{\theta}=|\eta{-}\sigma_i||{\lambda}^{(1)}_i|(1{-}\theta)^k.
\end{align}
Therefore, 
\begin{align}
\label{eqn:linconv}
F_i(\bL^{(k+1)}){-}F_i(\bL^{*})=(1{-}\theta)(F_i(\bL^{(k)}){-}F_i(\bL^{*})),
\end{align}
which is a linear convergence. Similarly, cases 7) and 8) also result in a linear convergence.

As for the general situation, $F_i(\bL^{(k)})$ can visit any of the $10$ cases at each iteration. Thus its convergence rate is dominated by the slowest case, which is a linear convergence. Summing up all the $(F_i(\bL^{(1)})-F_i(\bL^{*}))$ in \eqref{eqn:totaldiff}, $F(\bL^{(k)})$ achieves an overall linear convergence.

\end{proof}

Theorem \ref{thm:objdescent} indicates that LTPSS has a linear convergence rate and thus requires $O(\log(\frac{1}{\varepsilon}))$ iterations to achieve a convergence tolerance of $\varepsilon>0$. In each iteration, LTPSS requires $O(N^2)$, $O(N^3)$, $O(N^2)$, and $O(N^2)$ to conduct a gradient descent step, a singular value thresholding, a projection onto the feasible set, and a KM iteration, respectively. Hence the computational complexity for one iteration is $O(N^3)$, and the overall computational complexity of LTPSS is $O(N^3\log(\frac{1}{\varepsilon}))$. The computational complexities of LTP-CF and LTP-PP are both $O(N^3)$, since they are closed-form methods with one singular value decomposition. Nevertheless, LTPSS improves investing performance on LTP-CF and LTP-PP, which offsets the additional computational cost.

Figure \ref{fig:logconver} shows an example plot of $\log(F(\bL^{(k)})-F(\bL^{(k+1)}))$ for the proposed KM algorithm in one run that converges at the $92$-nd iteration. Since $\log(F(\bL^{(k)})-F(\bL^{(k+1)}))$ decreases in a step-wise way until convergence, the KM algorithm enjoys considerable constant-step descents. Other runs of the KM algorithm also have similar plots. In order to make comparison, we remove the constraint $\|L\|_2 \leqslant 1$ so that the general proximal gradient method could be applied, which may be the closest method to the KM algorithm. The plot of the proximal gradient algorithm is also shown in Figure \ref{fig:logconver}, which indicates a worse convergence rate than the KM algorithm.

\begin{figure} 
			\centering
\includegraphics[width=0.99\columnwidth]{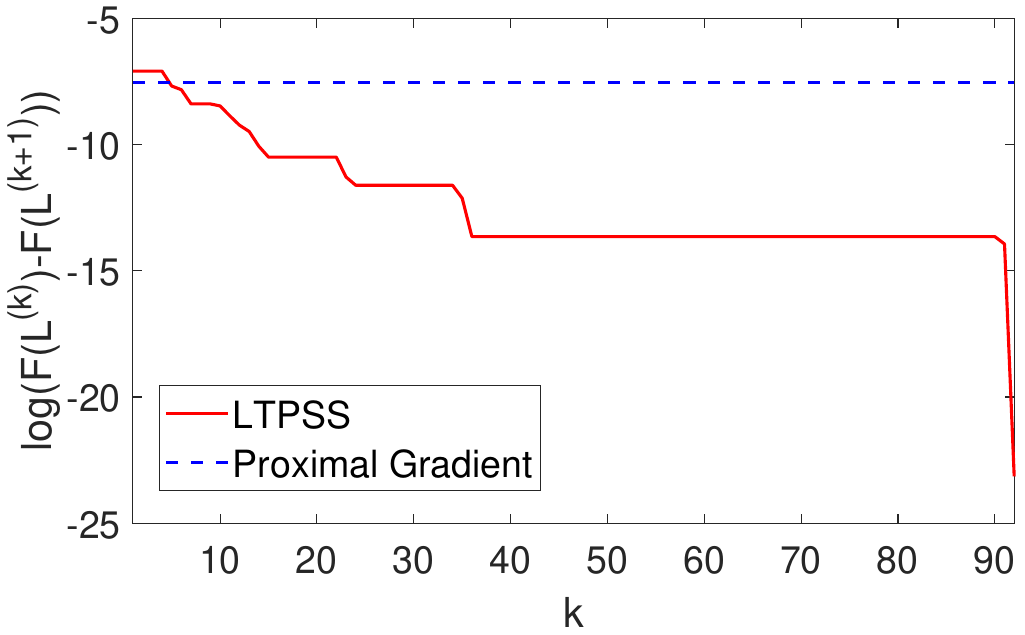}
		\caption{Convergence plots of the proposed KM algorithm and the general proximal gradient algorithm.}
		\label{fig:logconver}
\end{figure}

\section{Experimental Results}
\label{sec:experiment}
We follow the experimental framework of \cite{PPA,egrmvgap,desingijcai} to evaluate the performance of the proposed LTPSS. \textbf{The experimental settings and more experimental results are put in Supplementary \ref{sup:addexperiment}.}

\subsection{Mean Return}
\label{sec:meanret}

At the $t$-th trading time, a trading strategy determines an LTP $\hat{\bL}_{t+1}$ for the next trading time. Then the return of this strategy for the next trading time can be computed by $r_{t+1}:=\bS_t^\top \hat{\bL}_{t+1} \bR_{t+1}$. Suppose there are $\mT$ trading times in total. Then the mean return (MR) of this strategy is $\MR:=\frac{1}{\mT} \sum_{t=1}^\mT r_{t}$. It reflects the average investing gain of a trading strategy during the whole investment. The MRs of different trading strategies on the $7$ benchmark data sets are shown in Table \ref{tab:MR}. It indicates that LTP-PP performs well on the FF25 data sets, which are well interpreted by the Fama-French five factors. But it deteriorates on the other data sets collected in diverse financial circumstances. For example, LTP-PP suffers a negative MR on Stoxx50, which represents a financial market outside US. LTP-SF also performs badly with negative MRs on MSCI, Stoxx50, FOF, and FTSE100, which indicates that using only the own signal of each asset is ineffective. SDCP has similar bad performance as LTP-SF, since it cannot achieve full-reinvestment on each trade. mSSRM-PGA has relatively low MRs because it is a uni-linear trading strategy without an effective signal learning scheme. LTPSS achieves the best MRs on most data sets except MSCI where it is slightly worse than LTP-CF, hence more PPs may not necessarily lead to better performance. It achieves greater advantage over LTP-SF and LTP-PP on MSCI, Stoxx50, FOF, and FTSE100. For example, it performs about $1$, $2$, and $25$ times better than LTP-PP on MSCI, FOF, and FTSE100, respectively. These results indicate that LTPSS can extract key features from the prediction matrix and performs robustly in diverse financial circumstances.

\begin{table*}[!htb]
	\centering
\scalebox{0.83}{	\begin{tabular}{ccccccccc}
		\hline
Eval. & Strategy	&	FF25BM & FF25MEINV   & FF25MEOP &  MSCI  & Stoxx50  & FOF &  FTSE100     \\
\hline
\multirow{6}{*}{MR} &SDCP	 & 0.0121  & 0.0110 & 0.0111 & -0.0126 & -0.0087 & -0.0015 & -0.0795 \\
&LTP-SF	 &  0.0120 & 0.0107 & 0.0109 & -0.0170 & -0.0097 & -0.0017 & -0.0274 \\
 &LTP-CF	 & 0.0133 & 0.0137 & 0.0130 & \textbf{0.0090} & 0.0004 & 0.0022 & 0.0242 \\
 &LTP-PP	 & 0.0129 & 0.0133 & 0.0126 & 0.0041 & -0.0004 & 0.0012 & 0.0013 \\
  & mSSRM-PGA   &  0.0006   & 0.0006    & 0.0006    &  -0.0003   &  -0.0003  &  0.0006   &  0.0036   \\
 &\textbf{LTPSS}	 & \textbf{0.0133} & \textbf{0.0143} & \textbf{0.0134} & 0.0087 & \textbf{0.0012} & \textbf{0.0030} & \textbf{0.0273} \\
\hline
\multirow{6}{*}{SR} &SDCP	 & 0.1429  & 0.1522 & 0.1626 & -0.2702 & -0.0738 & -0.0254 & -0.2802 \\
 &LTP-SF	 &  0.1356 & 0.1394 & 0.1509 & -0.2918 & -0.0738 & -0.0280 & -0.0941 \\
 & LTP-CF	 & 0.2010 & 0.2590 & 0.2535 & 0.2257 & 0.0070 & 0.0659 & 0.1359 \\
 &LTP-PP	 & 0.2049 & 0.2642 & 0.1998 & 0.1743 & -0.0091 & 0.0312 & 0.0114 \\
  & mSSRM-PGA   &  0.0691   & 0.0683    &  0.0693    &  -0.0277   &  -0.0214  &  0.0413   &  0.1253   \\
 &\textbf{LTPSS}	 & \textbf{0.2049} & \textbf{0.2731} & \textbf{0.2619} & \textbf{0.2445} & \textbf{0.0198} & \textbf{0.0874} & \textbf{0.1521} \\
\hline
\multirow{5}{*}{IR}	  &SDCP	 & 0.1407  & 0.2103 &0.1645  & 0.0088 & \textbf{0.0922} & 0.0613 & -0.1957 \\
& LTP-CF &  0.1567 & 0.2203 & 0.1938 & 0.1064 & -0.0084 & 0.1060 & 0.1314 \\
	 & LTP-PP &  0.1489 & 0.2308 & 0.1715 & -0.0203 & -0.0365 & 0.1060 & -0.0468 \\
    & mSSRM-PGA     &  0.0852   & 0.0839    &  0.0854   &  -0.0368   &  -0.0329  &  0.0417   &  0.1320   \\
	 & \textbf{LTPSS} &  \textbf{0.1650} & \textbf{0.2418} & \textbf{0.2047} & \textbf{0.1366} & {0.0053} & \textbf{0.1409} & \textbf{0.1503} \\
\hline
	\multirow{5}{*}{$\alpha$}	 &SDCP	 & 0.0006  & 0.0009 & 0.0007 & 0.0001 & 0  & 0.0002 & 0 \\
	 & LTP-CF &  0.0063 & 0.0081 & 0.0067 & 0.0038 & -0.0005 & 0.0028 & 0.0229 \\
	 & LTP-PP &  0.0054 & 0.0079 & 0.0056 & -0.0004 & -0.0015 & 0.0022 & -0.0050 \\
   & mSSRM-PGA   &  0.0006   & 0.0005    &  0.0005   &  -0.0004   &   -0.0004 &  0.0005    & 0.0033    \\
	 & \textbf{LTPSS} &	 \textbf{0.0063} & \textbf{0.0087} & \textbf{0.0067} & \textbf{0.0043} & \textbf{0.0003} & \textbf{0.0036} & \textbf{0.0264} \\
\hline
		\multirow{5}{*}{p-val.}	 &SDCP	 &  0.0007  & 0 & 0.0001 & 0.4700 & 0.0284 & 0.2536 & 0.0263 \\
 & LTP-CF &	 0.0002 & 0 & 0 & 0.1807 & 0.5846 & 0.1260 & 0.1794 \\
	 & LTP-PP &	 0.0002 & 0 & 0 & 0.5694 & 0.8228 & 0.1260 & 0.8140 \\
  & mSSRM-PGA  &  0   &   0   &   0  &  0.1176   &  0.0201   &  0.0733   &  0.0002   \\
 & \textbf{LTPSS}	& 0.0001 & 0 & 0 & 0.1210 & 0.4468 & 0.0643 & 0.1471 \\
\hline
\multirow{6}{*}{MDD} &SDCP	 & 0.8700   &  0.7805  &  0.6695   &    0.7010 &  1.0000    & 0.4311   &  0.9685   \\
&LTP-SF	 	&  0.8858   & 0.8083   &  0.6857   &  0.7917   &  1.0000   &  0.4572   &   0.9895  \\
 &LTP-CF	 &  0.7131  &  0.5309  &  0.5188   &   0.3955  &  0.8944   &   0.1983  &  0.5986   \\
 &LTP-PP	&  0.6530  & 0.5078   &   0.5208  &   0.1453  &  0.7780   &  0.2188   &  0.4017   \\
  & mSSRM-PGA   &  0.6237   &  0.4834   &  0.4958   &  0.5294   &   0.8676  &   0.5454  &  0.4596   \\
  & \textbf{LTPSS}		& 0.6918   &  0.5184  & 0.5259    &  0.3295   &  0.8126   &  0.1894   &  0.5141   \\
\hline
	\end{tabular}}
	\caption{Experimental results of different trading strategies on $7$ benchmark data sets.}
	\label{tab:MR}
\end{table*}

\subsection{Sharpe Ratio}
\label{sec:SR}
The Sharpe ratio (SR, \cite{SHARPratio}) is a widely-used evaluating metric for the risk-adjusted return of a trading strategy. It can be calculated as follows in the framework of this paper: $\SR:=\frac{\MR-{r}_{\ff}}{\hat{s}({r}_{t})}$, where ${r}_{\ff}$ is the average risk-free return during the investment, and $\hat{s}(\cdot)$ is the sample standard deviation operator. The SR need not be annualized and this does not affect the assessment. ${r}_{\ff}$ is given on the FF25 data sets but not given on MSCI, Stoxx50, FOF, and FTSE100. Hence we set ${r}_{\ff}=0$ for the latter $4$ data sets. The SRs of different trading strategies are shown in Table \ref{tab:MR}. Again, SDCP, LTP-SF, and mSSRM-PGA have similar bad performance due to ineffective investing schemes. LTP-PP outperforms LTP-CF on FF25BM and FF25MEINV, but it still performs badly on Stoxx50 and FTSE100. It indicates that the first few PPs may not be robust enough to different financial circumstances.

LTPSS achieves the best SRs on all the data sets, especially on Stoxx50, FOF and FTSE100 where it keeps up robust performance. The advantage of LTPSS over LTP-CF and LTP-PP is large on Stoxx50. These results indicate that LTPSS is effective and robust in balancing investing return and risk.

\subsection{Information Ratio}
\label{sec:IR}
In finance, it is conventional to factorize the return of a trading strategy for in-depth analysis of its composition. Following \cite{PPA}, we adopt the following Fama-French five-factor model:
	\begin{align}
		\label{eqn:ff5factor}
{r}_{t}=&\alpha+\zeta_0 {r}_{SF,t}+\zeta_1 MKT_t+\zeta_2 SMB_t+\zeta_3 HML_t\nonumber\\
&+\zeta_4 RMW_t+\zeta_5 CMA_t+\varepsilon_t,
	\end{align}	
where ${r}_{t}$ and ${r}_{SF,t}$ denote the returns on the $t$-th trading time of a trading strategy and the LTP-SF strategy, respectively. MKT (market), SMB (size), HML (value), RMW (profitability), and CMA (investment) denote the Fama-French five factors. $\zeta_0\sim\zeta_5$ are the coefficients of the corresponding factors. $\varepsilon_t$ is the residue for this regression model. $\alpha$ \cite{portalpha0} can be seen as the pure return of the trading strategy excluding all the above-mentioned factors, which will be further investigated in Section \ref{sec:alphafact}. After conducting a regression for \eqref{eqn:ff5factor}, estimations for $\hat{\alpha}$ and $\hat{s}(\varepsilon_t)$ can be obtained. The information ratio (IR, \cite{inforatio2}) is another widely-used
risk-adjusted return: $\IR:=\frac{\hat{\alpha}}{\hat{s}(\varepsilon_t)}$. It can be seen as the pure SR hedging out all the above-mentioned factors. Note that the return of LTP-SF has also been hedged out in \eqref{eqn:ff5factor}, thus LTP-SF does not have IRs itself. As for other trading strategies, a positive IR can be roughly considered as a better performance than LTP-SF. Table \ref{tab:MR} shows the IRs of different trading strategies. LTP-PP performs well on the FF25 data sets but deteriorates on MSCI, Stoxx50, and FTSE100. In comparison, LTPSS outperforms LTP-CF and LTP-PP on all the data sets, especially on MSCI and Stoxx50 where it achieves effective and robust performance. These results also show its good ability in balancing return and risk.

\subsection{$\alpha$ Factor}
\label{sec:alphafact}
$\alpha$ represents the pure return of a trading strategy based on the market-neutral perspective, since it hedges out all the market-related factors. It measures the intrinsic return that is brought by the traded assets themselves instead of the market volatility. It is widely employed to evaluate the effectiveness of a trading strategy despite the market effect. Along with the regression for \eqref{eqn:ff5factor}, a left-tailed t-test can be conducted to obtain the p-value for the corresponding $\alpha$. Both $\alpha$s and p-values of different trading strategies are presented in Table \ref{tab:MR}. SDCP and mSSRM-PGA perform badly on all the data sets. Similar to the results of IR, LTP-PP performs badly on MSCI, Stoxx50, and FTSE100. In comparison, LTPSS achieves the best performance on all the data sets. Hence LTPSS is effective in achieving pure returns.

\subsection{Maximum Drawdown}
\label{sec:MDD}
The maximum drawdown (MDD, \cite{MDD}) measures the maximum percentage loss of wealth from a past peak to a past valley for a PO method, which corresponds to the worst case in the investment. It is mainly treated by a back-end risk control mechanism, such as setting stop loss orders. For a reference purpose, we present MDDs of different trading strategies in Table \ref{tab:MR}. LTPSS achieves competitive MDDs, roughly between those of LTP-PP and LTP-CF. Since LTPSS achieves the best performance in most cases with other evaluating metrics, it is an effective method considering both return and risk.

\section{Conclusion}
\label{sec:conclude}
In this work, we propose a linear trading position with sparse spectrum (LTPSS) that can explore a larger spectral region of the prediction matrix and extract key features from it. Our approach improves on the existing principal portfolio (PP) strategy that it allows for all the PPs with flexible spectral energies. We develop a Krasnosel'ski\u \i-Mann fixed-point algorithm to solve the proposed optimization model, which is able to handle the complicated geometric structure of the constraint and the nuclear norm regularization. It also guarantees a linear convergence and the descent property of the objective values. Experimental results show that LTPSS achieves competitive and robust investing performance in diverse financial circumstances according to some common evaluating metrics. 

A direct impact of LTPSS is to provide a novel trading strategy for the interdiscipline of machine learning and financial technology. It is highly interpretable and understandable in the general framework of finance, as well as tractable and reliable with theoretical guarantees. It can enrich the factor library for quantitative research, or reveal some intrinsic patterns of the asset pool. A possible limitation is that LTPSS may not be applied to nonlinear trading frameworks, which is partly due to the lack of such trading frameworks in finance. As for future works, we can extend LTPSS to more general factor trading architectures, or develop a nonlinear trading position that has a stronger signal processing ability. Exploring the relationship between PPs and conventional financial factors is also an interesting topic.


\section*{Acknowledgments}
This work is supported in part by the National Natural Science Foundation of China under grants 62176103 and 72173141, in part by the Guangdong Basic and Applied Basic Research Foundation under grant 2024A1515010749, and in part by the Science and Technology Planning Project of Guangzhou under grant 2024A04J9896.

\bibliographystyle{named}
\bibliography{bibfile}

\clearpage

\appendix
\onecolumn

\section{Supplementary Material}

\subsection{Proof of Theorem \ref{thm:nonexpan}}
\label{proof:nonexpan}

\begin{proof}
We need to prove that $\forall \bA,\bB\in \bbRNN$, $\|\sT(\bA)-\sT(\bB)\|_F\leqs \|\bA-\bB\|_F$. Since $\sG(\bA)-\sG(\bB)= \bA-\bB$, $\sG$ is a non-expansive operator. As for $\prox_{\beta\eta\|\cdot\|_*}$, it is obvious that $\beta\eta\|\cdot\|_*$ is a convex function because
\begin{align*}
&\|\theta\bA+(1-\theta)\bB\|_{*}\leqs \theta\|\bA\|_{*}+(1-\theta)\|\bB\|_{*}, \qquad\forall \bA,\bB\in \bbRNN,\ \forall \theta \in [0,1].
\end{align*}
Then according to \cite{moreau1965proximite}, $\prox_{\beta\eta\|\cdot\|_*}$ is a firmly non-expansive operator, which is also non-expansive. Next, Theorem \ref{thm:proj} indicates that $\proj_{\Omega}$ is also a non-expansive operator. Finally, $\sT$ is non-expansive as composed by three non-expansive operators.
\end{proof}

\subsection{Proof of Proposition \ref{prop:fixedpointexist}}
\label{proof:fixedpointexist}

\begin{proof}
Since the last component of $\sT$ is $\proj_{\Omega}$, we can restrict $\sT: \Omega\rightarrow \Omega$ as a mapping from $\Omega$ to itself. Proposition \ref{prop:conclobou} indicates that $\Omega$ is convex and compact, while Theorem \ref{thm:nonexpan} indicates that $\sT$ is $1$-Lipschitz continuous (non-expansive). Then it follows from the Schauder fixed-point theorem that $\sT$ has at least one fixed-point in $\Omega$. 
\end{proof}

\subsection{Proof of Theorem \ref{thm:KMtheorem}}
\label{proof:KMtheorem}
The general framework for the convergence of KM algorithms originates from \cite{KMalgoK}
and \cite{KMalgoM}. We follow their framework in our special problem. First, we take any fixed point $\bL^\bullet\in\mF_\Omega$ and examine its distances to two successive iterates $\bL^{(k)}$ and $\bL^{(k+1)}$. To do this, we need the following equality for the inner product:
\begin{align}
\label{eqn:innerequality}
2\tr((\bA-\bB)^\top (\bB-\bC))=\|\bA-\bC\|_F^2-\|\bA-\bB\|_F^2-\|\bB-\bC\|_F^2, \quad\forall \bA,\bB,\bC\in \bbRNN,
\end{align}	
which is easy to verify. Then,
\begin{align}
\label{eqn:lkp1fix}
&\|\bL^{(k+1)}-\bL^\bullet\|_F^2\nonumber\\
=&\|(1-\theta)\bL^{(k)}+\theta\sT(\bL^{(k)})-((1-\theta)+\theta)\bL^\bullet\|_F^2\nonumber\\
=&\|(1-\theta)(\bL^{(k)}-\bL^\bullet)+\theta(\sT(\bL^{(k)})-\bL^\bullet)\|_F^2\nonumber\\
=&\|(1-\theta)(\bL^{(k)}-\bL^\bullet)\|_F^2+\|\theta(\sT(\bL^{(k)})-\bL^\bullet)\|_F^2+2(1-\theta)\theta\tr((\bL^{(k)}-\bL^\bullet)^\top(\sT(\bL^{(k)})-\bL^\bullet))\nonumber\\
=&\|(1-\theta)(\bL^{(k)}-\bL^\bullet)\|_F^2+\|\theta(\sT(\bL^{(k)})-\bL^\bullet)\|_F^2\nonumber\\
&-(1-\theta)\theta(\|(\bL^{(k)}-\sT(\bL^{(k)})\|_F^2-\|\bL^{(k)}-\bL^\bullet\|_F^2-\|\bL^\bullet-\sT(\bL^{(k)})\|_F^2)\nonumber\\
=&(1-\theta)\|\bL^{(k)}-\bL^\bullet\|_F^2+\theta\|\bL^\bullet-\sT(\bL^{(k)})\|_F^2-(1-\theta)\theta\|(\bL^{(k)}-\sT(\bL^{(k)})\|_F^2\nonumber\\
=&(1-\theta)\|\bL^{(k)}-\bL^\bullet\|_F^2+\theta\|\sT(\bL^\bullet)-\sT(\bL^{(k)})\|_F^2-(1-\theta)\theta\|(\bL^{(k)}-\sT(\bL^{(k)})\|_F^2\nonumber\\
\leqs&(1-\theta)\|\bL^{(k)}-\bL^\bullet\|_F^2+\theta\|\bL^\bullet-\bL^{(k)}\|_F^2-(1-\theta)\theta\|(\bL^{(k)}-\sT(\bL^{(k)})\|_F^2\nonumber\\
=&\|\bL^{(k)}-\bL^\bullet\|_F^2-(1-\theta)\theta\|(\bL^{(k)}-\sT(\bL^{(k)})\|_F^2.
\end{align}	
The inequality holds since $\sT$ is non-expansive. Since $\theta\in (0,1)$, we have 
\begin{align}
\label{eqn:fejer}
\|\bL^{(k+1)}-\bL^\bullet\|_F^2\leqs \|\bL^{(k)}-\bL^\bullet\|_F^2, \quad \forall k\geqs 1.
\end{align}	
Moreover, $\bL^\bullet$ is arbitrary in $\mF_\Omega$. Hence $\{\bL^{(k)}\}_{k\geqs 1}$ is called a \emph{Fej\'er monotone} sequence \cite{bauschke2017convex} w.r.t. $\mF_\Omega$. It has a useful property that $\bL^{(k+1)}$ will not get strictly farther than $\bL^{(k)}$ from any point in $\mF_\Omega$. Therefore, $\{\bL^{(k)}\}_{k\geqs 1}$ is a bounded sequence.

Summing up both sides of \eqref{eqn:lkp1fix} from $k=1$ to $k=m$ yields
\begin{align}
\label{eqn:lkp1fix2}
&\|\bL^{(m+1)}-\bL^\bullet\|_F^2\leqs \|\bL^{(1)}-\bL^\bullet\|_F^2-(1-\theta)\theta\sum_{k=1}^m \|(\bL^{(k)}-\sT(\bL^{(k)})\|_F^2, \quad\forall m\geqs 1.
\end{align}	
Hence,
\begin{align}
\label{eqn:lkp1fix3}
&(1-\theta)\theta\sum_{k=1}^m \|(\bL^{(k)}-\sT(\bL^{(k)})\|_F^2\leqs \|\bL^{(1)}-\bL^\bullet\|_F^2, \quad\forall m\geqs 1.
\end{align}	
Letting $m\rightarrow \infty$ yields
\begin{align}
\label{eqn:lkp1fixinf}
&\sum_{k=1}^\infty \|(\bL^{(k)}-\sT(\bL^{(k)})\|_F^2\leqs \frac{1}{(1-\theta)\theta}\|\bL^{(1)}-\bL^\bullet\|_F^2<\infty.
\end{align}	
Therefore, $\|(\bL^{(k)}-\sT(\bL^{(k)})\|_F\rightarrow 0$ as $k\rightarrow \infty$. In other words, $(\bL^{(k)}-\sT(\bL^{(k)}))$ converges to $\bzer_{N\times N}$.

Since $\{\bL^{(k)}\}_{k\geqs 1}$ is a bounded sequence, it has at least one cluster point $\bL^*$. Then there exists a subsequence $\{\bL^{(k_j)}\}$ that converges to $\bL^*$. Since $\sT$ is continuous, 
\begin{align}
\label{eqn:Tcontin}
\lim_{j\rightarrow \infty} \sT(\bL^{(k_j)})=\sT(\lim_{j\rightarrow \infty} \bL^{(k_j)})=\sT(\bL^*).
\end{align}	
Since $(\bL^{(k)}-\sT(\bL^{(k)}))$ converges to $\bzer_{N\times N}$, $(\bL^{(k_j)}-\sT(\bL^{(k_j)}))$ converges to $\bzer_{N\times N}$. Then we have
\begin{align}
\label{eqn:Tcontin2}
\lim_{j\rightarrow \infty}(\bL^{(k_j)}-\sT(\bL^{(k_j)}))=\lim_{j\rightarrow \infty}\bL^{(k_j)}-\lim_{j\rightarrow \infty}\sT(\bL^{(k_j)})=\bL^*-\sT(\bL^*)=\bzer_{N\times N}.
\end{align}	
Hence $\bL^*\in\mF_\Omega$. Note that this cluster point $\bL^*$ is arbitrary, thus the above deduction implies that every cluster point $\bL^*\in\mF_\Omega$.

The last step is to prove that $\{\bL^{(k)}\}_{k\geqs 1}$ can have at most one cluster point. If not, denote $\bL^*$ and $\bL^\bullet$ as two distinct cluster points of $\{\bL^{(k)}\}_{k\geqs 1}$. Namely, $\bL^{(k_j)}\rightarrow\bL^*$ and $\bL^{(k_m)}\rightarrow\bL^\bullet$. It follows from the above deductions that $\bL^*, \bL^\bullet\in \mF_\Omega$. Then for any $\bL^\circ\in \mF_\Omega$ (can be $\bL^*, \bL^\bullet$ or others), $\{\|\bL^{(k)}-\bL^\circ\|_F^2\}_{k\geqs 1}$ is a monotonically non-increasing sequence with a lower bound $0$ based on \eqref{eqn:fejer}. Then $\{\|\bL^{(k)}-\bL^\circ\|_F^2\}_{k\geqs 1}$ converges to a non-negative real number, and so does its subsequences:
\begin{align}
\label{eqn:onecluster}
\lim_{j\rightarrow \infty} \|\bL^{(k_j)}-\bL^\circ\|_F^2 =\lim_{m\rightarrow \infty} \|\bL^{(k_m)}-\bL^\circ\|_F^2,\quad \forall \bL^\circ\in \mF_\Omega.
\end{align}	
Expanding both sides of \eqref{eqn:onecluster} yields
\begin{align}
\label{eqn:onecluster2}
\lim_{j\rightarrow \infty} \|\bL^{(k_j)}\|_F^2- \lim_{m\rightarrow \infty} \|\bL^{(k_m)}\|_F^2=2\tr((\bL^*-\bL^\bullet)^\top\bL^\circ),\quad \forall \bL^\circ\in \mF_\Omega.
\end{align}	
The left side of \eqref{eqn:onecluster2} is a constant no matter what $\bL^\circ\in \mF_\Omega$ is chosen. Hence, inserting $\bL^\circ=\bL^*$ and $\bL^\circ=\bL^\bullet$ into the right side of \eqref{eqn:onecluster2} keeps the equality
\begin{align}
\label{eqn:onecluster3}
2\tr((\bL^*-\bL^\bullet)^\top\bL^*)=2\tr((\bL^*-\bL^\bullet)^\top\bL^\bullet)\quad \Leftrightarrow \quad  \|\bL^*-\bL^\bullet\|_F^2=0.
\end{align}	
Therefore, $\bL^*=\bL^\bullet$ and $\{\bL^{(k)}\}_{k\geqs 1}$ has exactly one cluster point. This actually means that $\{\bL^{(k)}\}_{k\geqs 1}$ converges to only one fixed point in $\mF_\Omega$.

\subsection{$10$ Cases for Each Summand in (\ref{eqn:totaldiff3})}
\label{sup:tencases}

1) $ {\lambda}^{(k)}_i+\beta\sigma_i\geqs 0$, ${\lambda}^{(k)}_i\geqs 0$ and $\sigma_i\geqs\eta$: \eqref{eqn:svdmathinduce3} indicates that ${\gamma}^{(k)}_i-{\lambda}^{(k)}_i=\min\{ \beta(\sigma_i-\eta), 1\}\geqs 0$ and ${\lambda}^{(k+1)}_i\geqs 0$. Then 
\begin{align}
\label{eqn:Fdescent1}
&({\gamma}^{(k)}_i-{\lambda}^{(k)}_i)\theta\sigma_i+\eta(|{\lambda}^{(k)}_i|-|{\lambda}^{(k+1)}_i|)\nonumber\\
=&\min\{ \beta(\sigma_i{-}\eta), 1\}{\cdot}\theta(\sigma_i{-}\eta)\geqs \beta\theta(\sigma_i{-}\eta)^2\geqs 0.
\end{align}

2) $ {\lambda}^{(k)}_i+\beta\sigma_i\geqs 0$, ${\lambda}^{(k)}_i< 0$, ${\lambda}^{(k+1)}_i\geqs 0$ and $\sigma_i\geqs\eta$: \eqref{eqn:svdmathinduce3} indicates that ${\gamma}^{(k)}_i-{\lambda}^{(k)}_i=\min\{ \beta(\sigma_i-\eta), 1\}\geqs 0$. Then 
\begin{align}
\label{eqn:Fdescent2}
&({\gamma}^{(k)}_i-{\lambda}^{(k)}_i)\theta\sigma_i+\eta(|{\lambda}^{(k)}_i|-|{\lambda}^{(k+1)}_i|)\nonumber\\
=& ({\gamma}^{(k)}_i-{\lambda}^{(k)}_i)\theta\sigma_i+\eta(-{\lambda}^{(k)}_i-{\lambda}^{(k+1)}_i)   \nonumber\\
>& ({\gamma}^{(k)}_i-{\lambda}^{(k)}_i)\theta\sigma_i+\eta({\lambda}^{(k)}_i-{\lambda}^{(k+1)}_i)   \nonumber\\
=&\min\{ \beta(\sigma_i{-}\eta), 1\}{\cdot}\theta(\sigma_i{-}\eta)\geqs \beta\theta(\sigma_i{-}\eta)^2\geqs 0.
\end{align}

3) $ {\lambda}^{(k)}_i+\beta\sigma_i\geqs 0$, ${\lambda}^{(k)}_i< 0$, ${\lambda}^{(k+1)}_i< 0$ and $\sigma_i\geqs\eta$: \eqref{eqn:svdmathinduce3} indicates that ${\gamma}^{(k)}_i-{\lambda}^{(k)}_i=\min\{ \beta(\sigma_i-\eta), 1\}\geqs 0$. Then 
\begin{align}
\label{eqn:Fdescent3}
&({\gamma}^{(k)}_i-{\lambda}^{(k)}_i)\theta\sigma_i+\eta(|{\lambda}^{(k)}_i|-|{\lambda}^{(k+1)}_i|)\nonumber\\
=& ({\gamma}^{(k)}_i-{\lambda}^{(k)}_i)\theta\sigma_i+\eta(-{\lambda}^{(k)}_i+{\lambda}^{(k+1)}_i)   \nonumber\\
=&\min\{ \beta(\sigma_i{-}\eta), 1\}\cdot\theta(\sigma_i{+}\eta)\geqs \beta\theta(\sigma_i{-}\eta)^2\geqs 0.
\end{align}

4) $ {\lambda}^{(k)}_i+\beta(\sigma_i-\eta)\geqs 0$, ${\lambda}^{(k+1)}_i\geqs 0$ and $\sigma_i<\eta$: \eqref{eqn:svdmathinduce3} indicates that ${\lambda}^{(k)}_i\geqs 0$ and ${\gamma}^{(k)}_i-{\lambda}^{(k)}_i= \beta(\sigma_i-\eta)$. Then
\begin{align}
\label{eqn:Fdescent4}
&({\gamma}^{(k)}_i-{\lambda}^{(k)}_i)\theta\sigma_i+\eta(|{\lambda}^{(k)}_i|-|{\lambda}^{(k+1)}_i|)\nonumber\\
=&({\gamma}^{(k)}_i-{\lambda}^{(k)}_i)\theta\sigma_i+\eta({\lambda}^{(k)}_i-{\lambda}^{(k+1)}_i)\nonumber\\
=&\theta\beta(\sigma_i-\eta)^2\geqs 0.
\end{align}

5) $ {\lambda}^{(k)}_i+\beta(\sigma_i-\eta)\geqs 0$, ${\lambda}^{(k+1)}_i< 0$ and $\sigma_i<\eta$: \eqref{eqn:svdmathinduce3} indicates that ${\lambda}^{(k)}_i\geqs 0$ and ${\gamma}^{(k)}_i-{\lambda}^{(k)}_i= \beta(\sigma_i-\eta)$. Then
\begin{align}
\label{eqn:Fdescent5}
&({\gamma}^{(k)}_i-{\lambda}^{(k)}_i)\theta\sigma_i+\eta(|{\lambda}^{(k)}_i|-|{\lambda}^{(k+1)}_i|)\nonumber\\
=&2\eta{\lambda}^{(k)}_i + ({\gamma}^{(k)}_i-{\lambda}^{(k)}_i)\theta\sigma_i+\eta(-{\lambda}^{(k)}_i+{\lambda}^{(k+1)}_i)\nonumber\\
=&2\eta{\lambda}^{(k)}_i + \theta\beta(\sigma_i-\eta)(\sigma_i+\eta) \nonumber\\
\geqs&2\eta\beta(\eta-\sigma_i) + \theta\beta(\sigma_i-\eta)(\sigma_i+\eta) \nonumber\\
=&\beta(\eta-\sigma_i)[(2-\theta)\eta-\theta\sigma_i]\nonumber\\
\geqs& \beta(\eta-\sigma_i)(\theta\eta-\theta\sigma_i)=\beta\theta(\sigma_i-\eta)^2\geqs 0.
\end{align}
Since $\eta>\sigma_i$ and $\theta\in (0,1)$, we have $\eta>\frac{\sigma_i}{2/\theta-1}$. Hence $[(2-\theta)\eta-\theta\sigma_i]>0$ and all the inequalities in \eqref{eqn:Fdescent5} hold.

6) $  0\leqs {\lambda}^{(k)}_i+\beta\sigma_i<\beta\eta$, ${\lambda}^{(k)}_i\geqs 0$ and $\sigma_i<\eta$:  \eqref{eqn:svdmathinduce3} indicates that ${\gamma}^{(k)}_i=0$ and ${\lambda}^{(k+1)}_i=(1-\theta){\lambda}^{(k)}_i\geqs 0$. Then
\begin{align}
\label{eqn:Fdescent6}
&({\gamma}^{(k)}_i-{\lambda}^{(k)}_i)\theta\sigma_i+\eta(|{\lambda}^{(k)}_i|-|{\lambda}^{(k+1)}_i|)\nonumber\\
=&({\gamma}^{(k)}_i-{\lambda}^{(k)}_i)\theta\sigma_i+\eta({\lambda}^{(k)}_i-{\lambda}^{(k+1)}_i)\nonumber\\
=&(\eta-\sigma_i)\theta{\lambda}^{(k)}_i \geqs 0.
\end{align}

7) $  0\leqs {\lambda}^{(k)}_i+\beta\sigma_i<\beta\eta$, ${\lambda}^{(k)}_i< 0$ and $\sigma_i<\eta$:  \eqref{eqn:svdmathinduce3} indicates that ${\gamma}^{(k)}_i=0$ and ${\lambda}^{(k+1)}_i=(1-\theta){\lambda}^{(k)}_i< 0$. Then
\begin{align}
\label{eqn:Fdescent7}
&({\gamma}^{(k)}_i-{\lambda}^{(k)}_i)\theta\sigma_i+\eta(|{\lambda}^{(k)}_i|-|{\lambda}^{(k+1)}_i|)\nonumber\\
=&({\gamma}^{(k)}_i-{\lambda}^{(k)}_i)\theta\sigma_i+\eta(-{\lambda}^{(k)}_i+{\lambda}^{(k+1)}_i)\nonumber\\
=&-(\eta+\sigma_i)\theta{\lambda}^{(k)}_i \geqs 0.
\end{align}

8) ${\lambda}^{(k)}_i+\beta\sigma_i<0\leqs {\lambda}^{(k)}_i+\beta(\sigma_i+\eta)$:  \eqref{eqn:svdmathinduce3} indicates that ${\gamma}^{(k)}_i=0$, ${\lambda}^{(k)}_i<0$ and ${\lambda}^{(k+1)}_i=(1-\theta){\lambda}^{(k)}_i< 0$. Then
\begin{align}
\label{eqn:Fdescent8}
&({\gamma}^{(k)}_i-{\lambda}^{(k)}_i)\theta\sigma_i+\eta(|{\lambda}^{(k)}_i|-|{\lambda}^{(k+1)}_i|)\nonumber\\
=&-(\eta+\sigma_i)\theta{\lambda}^{(k)}_i \geqs 0.
\end{align}

9) ${\lambda}^{(k)}_i+\beta(\sigma_i+\eta)<0$ and ${\lambda}^{(k+1)}_i\geqs 0$:  \eqref{eqn:svdmathinduce3} indicates that ${\lambda}^{(k)}_i<0$ and ${\gamma}^{(k)}_i-{\lambda}^{(k)}_i=\beta(\sigma_i+\eta)\geqs 0$. Then
\begin{align}
\label{eqn:Fdescent9}
&({\gamma}^{(k)}_i-{\lambda}^{(k)}_i)\theta\sigma_i+\eta(|{\lambda}^{(k)}_i|-|{\lambda}^{(k+1)}_i|)\nonumber\\
=& ({\gamma}^{(k)}_i-{\lambda}^{(k)}_i)\theta\sigma_i+\eta({\lambda}^{(k)}_i-{\lambda}^{(k+1)}_i)-2\eta{\lambda}^{(k)}_i\nonumber\\
>& \beta(\sigma_i+\eta)(\theta\sigma_i-\theta\eta)+2\eta\beta(\sigma_i+\eta)\nonumber\\
=& [\theta\sigma_i+(2-\theta)\eta]\beta(\sigma_i+\eta)\nonumber\\
\geqs& (\theta\sigma_i+\theta\eta)\beta(\sigma_i+\eta)\geqs\theta\beta(\sigma_i-\eta)^2 \geqs 0.
\end{align}
The second inequality holds since $\theta<1$.

10) ${\lambda}^{(k)}_i+\beta(\sigma_i+\eta)<0$ and ${\lambda}^{(k+1)}_i< 0$:  \eqref{eqn:svdmathinduce3} indicates that ${\lambda}^{(k)}_i<0$ and ${\gamma}^{(k)}_i-{\lambda}^{(k)}_i=\beta(\sigma_i+\eta)\geqs 0$. Then
\begin{align}
\label{eqn:Fdescent10}
&({\gamma}^{(k)}_i-{\lambda}^{(k)}_i)\theta\sigma_i+\eta(|{\lambda}^{(k)}_i|-|{\lambda}^{(k+1)}_i|)\nonumber\\
=& ({\gamma}^{(k)}_i-{\lambda}^{(k)}_i)\theta\sigma_i+\eta(-{\lambda}^{(k)}_i+{\lambda}^{(k+1)}_i)\nonumber\\
=& \theta\beta(\sigma_i+\eta)^2\geqs\theta\beta(\sigma_i-\eta)^2\geqs 0.
\end{align}

\subsection{Experimental Settings and Additional Experimental Results}
\label{sup:addexperiment}

\subsubsection{Experimental Settings}
\label{sup:expset}

First, we introduce the $7$ benchmark data sets to be used in the experiments, shown in Table \ref{tab:data}. The first $3$ data sets are $25$ Fama-French portfolios collected from the Kenneth R. French's Data Library\footnote{\url{http://mba.tuck.dartmouth.edu/pages/faculty/ken.french/data_library.html}} \cite{PPA}. They are constructed by double-sorting U.S. stocks by their size and book-to-market (BM), size (ME) and investment(INV), and size (ME) and operating profits (OP), respectively. These portfolios can be well explained by the Fama-French five-factor model \cite{FF5fact}. Therefore, \cite{PPA} also calculate and exploit the corresponding Fama-French five factors for these $3$ data sets to make full use of their methods. Every $20$ consecutive daily returns are accumulated as $1$ monthly return (public holidays considered). Then the return for the $\tau$-th month is used as the signal $\bS_\tau$ for the $(\tau+1)$-th month. Such a lagged monthly return is a strong positive predictor of subsequent monthly returns based on a wide range of empirical investigations \cite{moemp2,moemp,PPA}. Then the prediction matrix $\hat{\bPi}_t$ is estimated by \eqref{eqn:piestimate} with the window size $T=120$. 

To assess the extendability and robustness of the trading methods, we also employ $4$ more data sets with diverse profiles in the literature. MSCI \cite{olpsjmlr} contains $24$ indices of $24$ countries around the world that constitute the MSCI World Index\footnote{\url{http://www.mscibarra.com}}. EURO Stoxx50 \cite{pubdata2} contains $49$ constituent stocks from the EURO Stoxx50 Index. FOF \cite{SPOLC} contains $24$ mutual funds in China. FTSE100 \cite{pubdata2} contains weekly data of $83$ constituent stocks from the Financial Times Stock Exchange 100 Index. For these data sets, the corresponding Fama-French five factors are unavailable. Hence we follow \cite{olpsjmlr,SSPO,SPOLC} to use the uniform-buy-and-hold strategy as the market baseline. Besides, MSCI, Stoxx50, and FOF have insufficient observations for a window size of $T=120$ with $20$ daily returns for each trading time. Hence we accumulate $5$ daily returns to form a weekly return (normally there are $5$ trading days for $1$ week). This scheme can assess the extendability and robustness of a trading strategy to weekly trading. As for FTSE100, we accumulate $4$ weekly returns to form a monthly return and conduct monthly trading.

\begin{table*}[!htb]
	\centering
\scalebox{0.9}{	\begin{tabular}{ccccccc}
		\hline
		Dataset & Region  & Asset Type &  Time  & Periods  & Frequency &  Assets     \\
		\hline
		FF25BM & US & Portfolio  & $1/Jul./1963\sim 31/Dec./2019$ &$14223$ & Daily  & $25$ \\
		FF25MEINV  & US & Portfolio  & $1/Jul./1963\sim 31/Dec./2019$ &$14223$ & Daily  & $25$ \\
		FF25MEOP  & US & Portfolio  & $1/Jul./1963\sim 31/Dec./2019$ &$14223$ & Daily  & $25$ \\
\hline
		MSCI & Global & Index  & $1/Apr./2006\sim 31/Mar./2010$ &$1043$ & Daily  & $24$ \\
		EURO Stoxx50 & EU & Stock & $22/May/2001\sim 11/Apr./2016$ & $3885$ & Daily & $49$ \\
		FOF& CN & Fund & $4/Jan./2013\sim 29/Dec./2017$ & $1215$ & Daily & $24$\\
		FTSE100 & UK & Stock  & $11/Jul./2002\sim 11/Apr./2016$ &$717$ & Weekly  & $83$ \\
		\hline
	\end{tabular}}
	\caption{Profiles of $7$ benchmark data sets.}
	\label{tab:data}
\end{table*}

Five trading strategies are taken into comparisons: 1) The surrogate model semi-definite conic programming (SDCP) defined in \eqref{eqn:LTPSSsurro}. Both absolute solution tolerance and constraint tolerance are set as $1e-6$. Note that the equality constraint $\|L\|_2 = 1$ cannot be achieved due to the mechanism of interior-point method, thus its investing performance is affected. 2) The simple factor $\hat{\bL}_{SF}$ in \eqref{eqn:simfact}. 3) The closed-form solution $\hat{\bL}_{CF}$ in \eqref{eqn:ltpcf}. It is actually formed by including all the PPs of the prediction matrix in \eqref{eqn:ltppp}. 4) $\hat{\bL}_{PP}:=\sum_{n=1}^3 \bu_n \bv_n^\top$ in \eqref{eqn:ltppp} (the first $3$ PPs), which is the best LTP-PP indicated and used by \cite{PPA}. 5) $m$-sparse Sharpe ratio maximization with the proximal gradient algorithm (mSSRM-PGA, \cite{MSSRM}), which is not a PCA based method and belongs to the broader field of portfolio optimization. Its default parameter settings are used in the experiments.

For the proposed LTPSS, we empirically set $\bL^{(0)}:=\hat{\bL}_{PP}$, $\beta=100$, $\eta=0.001$, and $\theta=0.9999$, which are consistent with the convergence criterion of the proposed algorithm. A larger $\theta$ generally leads to a faster convergence, but we could not directly set $\theta=1$ because this violates the condition $0<\theta<1$ of KM theorem (Theorem \ref{thm:KMtheorem}), which might lead to failure in convergence. Hence we set $\theta=0.9999$, which is close enough to $1$. We further examine the sensitivity of LTPSS to the regularization parameter $\eta$ within a $20\%$ range on the FF25BM data set. Sharpe ratio (SR) results in Table \ref{tab:SRapp} show that LTPSS gets stable when $\eta\in [0.001,0.0012]$.

\begin{table*}[!htb]
	\centering
	\begin{tabular}{cccccc}
		\hline
$\eta$   & 0.0008    & 0.0009   & 0.001   & 0.0011  &  0.0012	  \\
\hline
SR      &   0.1987    &  0.2013  & 0.2049   &  0.2041 & 0.2045   \\
\hline
	\end{tabular}
	\caption{Sharpe ratio (SR) of LTPSS w.r.t. the regularization parameter $\eta$ on FF25BM.}
	\label{tab:SRapp}
\end{table*}

The $7$ benchmark data sets are collected from major financial markets in the world and could be represented by the Fama-French factors to certain degrees, which might favor the above linear trading position methods. When it comes to other market conditions that cannot be represented by the Fama-French factors, such linear trading position methods might need new effective factors to keep up the change.

\subsubsection{Unstable Performance for Principal Portfolios}
\label{sup:unstablepp}
To look into the unstable performance for PPs, we calculate the mean returns of the first $10$ PPs, shown in Table \ref{tab:MRmore}. It indicates that the mean returns of the first $3$ PPs follow the same descending order as that of the PPs themselves on the $3$ FF25 data sets, which well fits the theory of \cite{PPA}. However, this pattern is broken in the $4$-th to $10$-th PPs on the $3$ FF25 data sets, as well as in all the $10$ PPs on the rest $4$ data sets. For example, there are even negative mean returns of the first $3$ PPs on Stoxx50, FOF and FTSE100. Moreover, there are larger mean returns of the PPs with higher orders than those with lower orders on the $3$ FF25 data sets. These results indicate that the investing performance is unstable while directly relying on several PPs.

\begin{table*}[!htb]
	\centering
\scalebox{0.9}{	\begin{tabular}{cccccccc}
		\hline
 Order	&	FF25BM & FF25MEINV   & FF25MEOP &  MSCI  & Stoxx50  & FOF &  FTSE100     \\
\hline
	1&	 0.009372 & 0.008423 & 0.006274 & 0.005083 & -0.000697 & 0.001075 & -0.001133 \\
	2	& 0.002235 & 0.003733 & 0.004856 & 0.001988 & 0.000289 & -0.002195 & 0.001540 \\
	3	& 0.001336 & 0.001135 & 0.001506 & 0.002758 & 0.000000 & -0.000692 & 0.000864 \\
	4	& -0.000009 & 0.000660 & 0.000490 & -0.000200 & -0.000352 & 0.000544 & 0.007757 \\
	5	& -0.000039 & 0.000034 & 0.000468 & -0.000186 & 0.000922 & -0.000047 & 0.004428 \\
	6	& 0.000143 & 0.000351 & -0.000035 & 0.000972 & -0.000415 & 0.001004 & -0.004271 \\
	7	& 0.000133 & -0.000062 & -0.000056 & 0.001688 & -0.000206 & -0.000337 & 0.002073 \\
	8	& 0.000054 & 0.000208 & 0.000019 & -0.000460 & -0.000001 & 0.000543 & 0.003393 \\
	9	& 0.000125 & -0.000032 & -0.000167 & 0.001474 & 0.000105 & 0.000510 & -0.003453 \\
	10	& -0.000083 & -0.000022 & -0.000068 & 0.000147 & -0.000648 & 0.000087 & 0.002064 \\
\hline
	\end{tabular}}
	\caption{Mean returns of the first $10$ principal portfolios on $7$ benchmark data sets.}
	\label{tab:MRmore}
\end{table*}

\subsubsection{Factor Exposure}
\label{sec:factexpo}
Similar to \cite{PPA}, we investigate the financial interpretation of different trading strategies by examining their exposure (coefficients) on different factors in the FF 5-factor model \eqref{eqn:ff5factor}. Note that most of the metrics here (except $\alpha$) do not evaluate investing performance, but just represent different investing logics or habits. First of all, LTP-CF has the lowest $R^2$, while LTP-PP and LTPSS have similar higher $R^2$s. It indicates that LTP-CF is the least interpretatable by the FF 5-factor model among the three, probably because its number of PPs is much larger than the number of FF factors (25 vs. 5). Similarly, LTPSS has a slightly lower $R^2$ than LTP-PP, because the former additionally includes sparse spectrum in its position. 

As for specific factor exposure, all the three strategies have statistically significant exposure on $\alpha$ and ${r}_{SF}$, since they are all factor-based trading strategies and significantly outperform the simple factor LTP-SF. However, LTP-CF and LTPSS have significant exposure on RMW, while LTP-PP has significant exposure on HML instead. It indicates that LTP-CF and LTPSS take stronger negative correlation with the robust-minus-weak factor, while LTP-PP takes stronger positive correlation with the high-minus-low factor. It reveals a small difference among the three methods in the investing focus.

\begin{table}[htbp]
	\centering
	\begin{tabular}{ccccccc}
	\toprule
		 &   \multicolumn{2}{c}{LTP-CF}  &   \multicolumn{2}{c}{LTP-PP}     &  \multicolumn{2}{c}{LTPSS}  \\
	\hline
Factor	 & Coefficient & p-val. & Coefficient & p-val. & Coefficient & p-val.     \\
	\hline
$\alpha$	 &  0.0063	 &  0.0004 &   0.0054 & 0.0008 & 0.0063 & 0.0002 \\
${r}_{SF}$	 &  0.5863	 &  0	 &   0.5812 & 0 & 0.5823 & 0 \\
MKT	 &  0.0076	 &  0.8611	 &   0.0427 & 0.2779 & 0.0047 & 0.9091 \\
SMB	 &  0.0905	 &  0.1683	 &   0.0314 & 0.5981 & 0.0964 & 0.1215 \\
HML	 & 	0.1262   &  0.1182	 &    0.2226 & 0.0024 & 0.1257 & 0.1006 \\
RMW	 &  -0.1863	 &  0.0269	 &   -0.1244 & 0.1028 & -0.1747 & 0.0285 \\
CMA	 &  -0.0187	 &  0.8802	 &   -0.0332 & 0.7683 & -0.0292 & 0.8045 \\
$R^2$ & 	0.6323 &    0   &  0.6677 & 0 & 0.6537 & 0 \\
	\bottomrule
	\end{tabular}
	\caption{Factor exposure of different trading strategies on FF25BM.}
\label{tb:factexpo}
\end{table}

\subsubsection{Results for Different Numbers of Principal Portfolios}
\label{sup:differentpps}
We have conducted experimental results of LTP-PP and LTPSS on mean returns, Sharpe ratios, information ratios and $\alpha$ factors for different numbers of principal portfolios, shown in Figures \ref{fig:MRmore}$\sim$\ref{fig:alphamore}. The number of principal portfolios means the parameter $l$ in \eqref{eqn:ltppp} and the corresponding $\hat{\bL}_{PP}$ used as the initial $\bL^{(0)}$ for the proposed algorithm of LTPSS. We examine $l\in [1,5]$ where $l=3$ is the recommended setting by \cite{PPA}. The results show that LTP-PP has unstable and sometimes bad performance, while LTPSS achieves robust and good performance on all the metrics and all the data sets. It indicates that LTPSS is an effective and reliable method for key spectral feature extraction in signal-based trading.

	\begin{figure*}[!htb]
		\centering
		\subfloat[FF25BM]{
			\centering
			\includegraphics[width=0.33\textwidth]{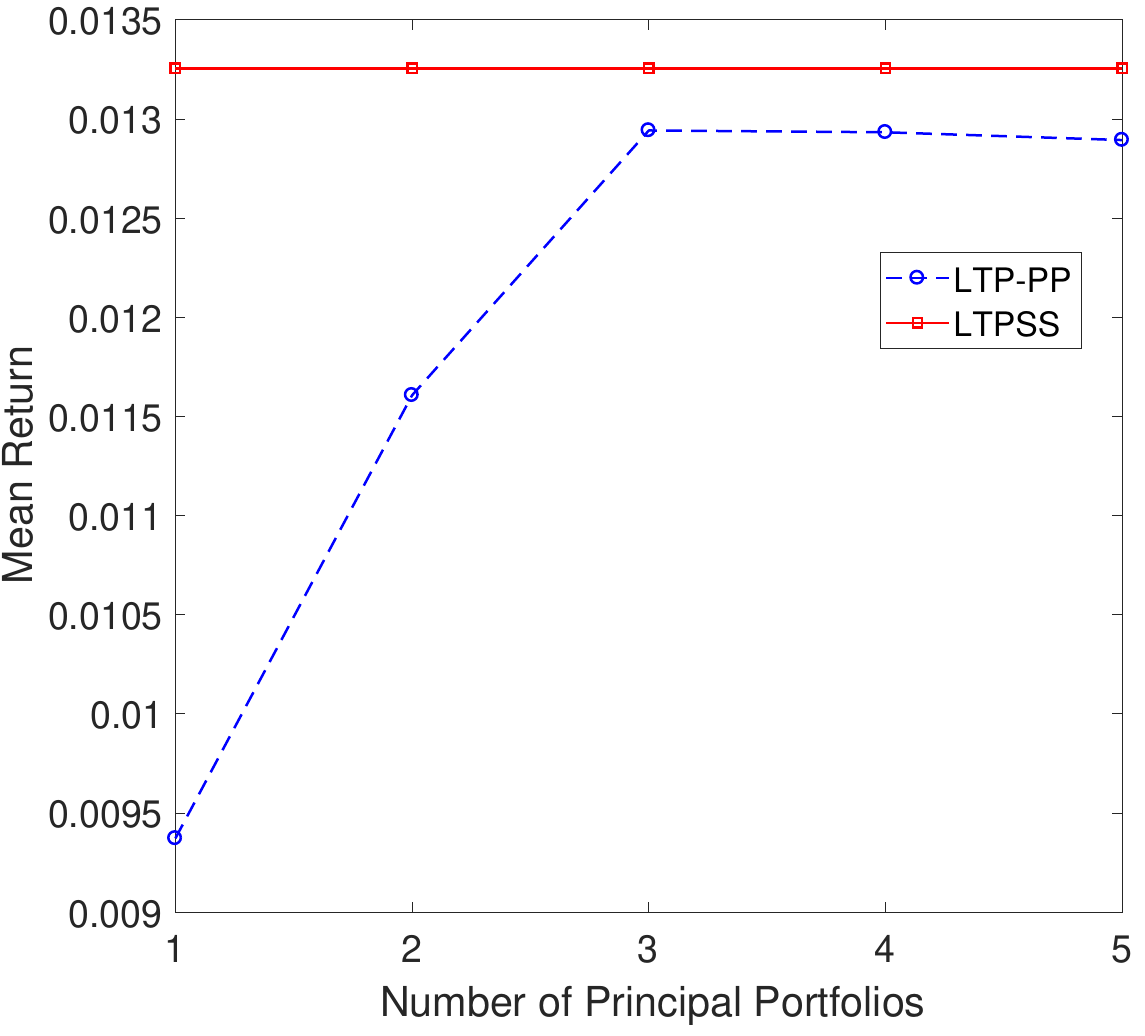}}
		\subfloat[FF25MEINV]{
			\centering
			\includegraphics[width=0.33\textwidth]{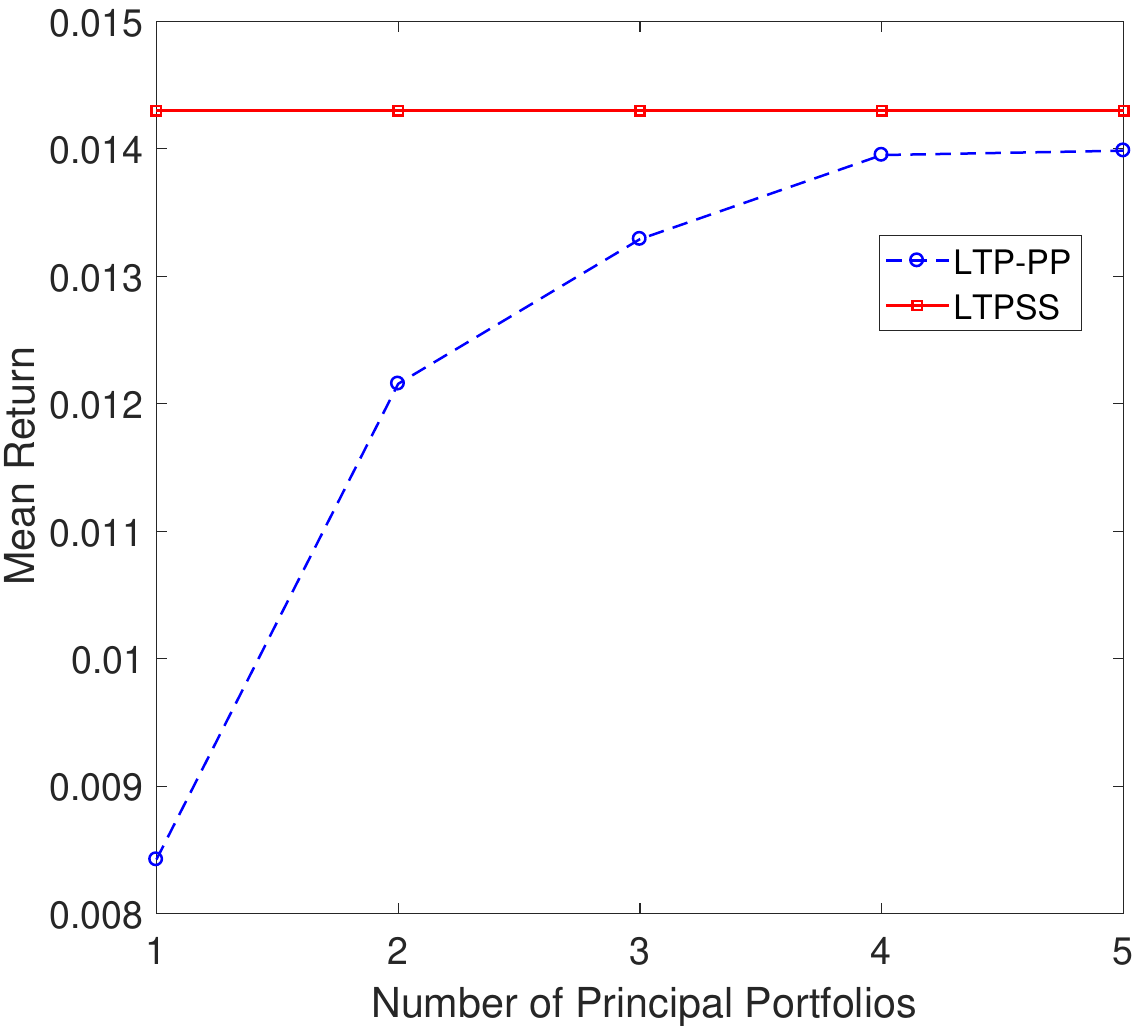}}
		\subfloat[FF25MEOP]{
			\centering
			\includegraphics[width=0.33\textwidth]{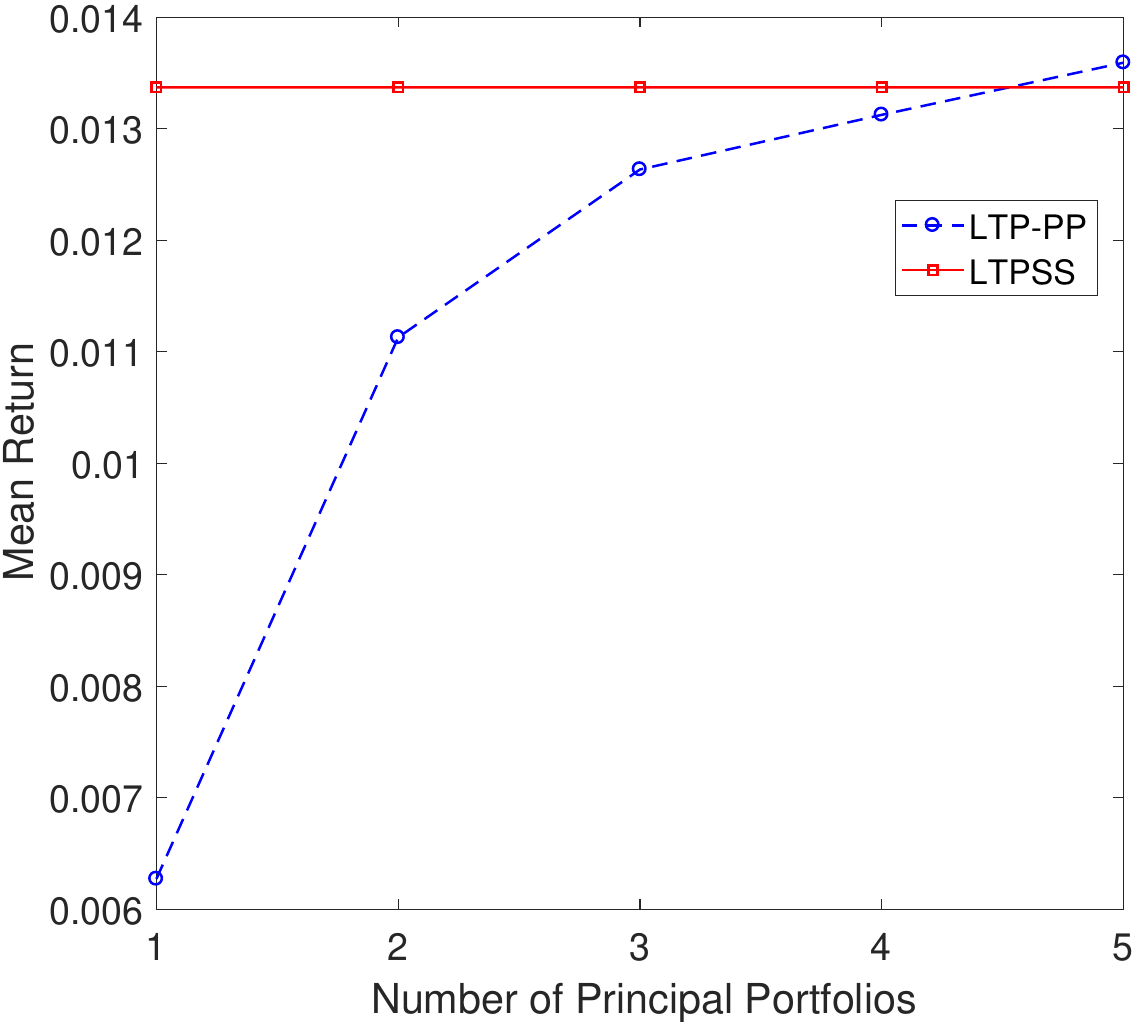}}\\
		\subfloat[MSCI]{
			\centering
			\includegraphics[width=0.33\textwidth]{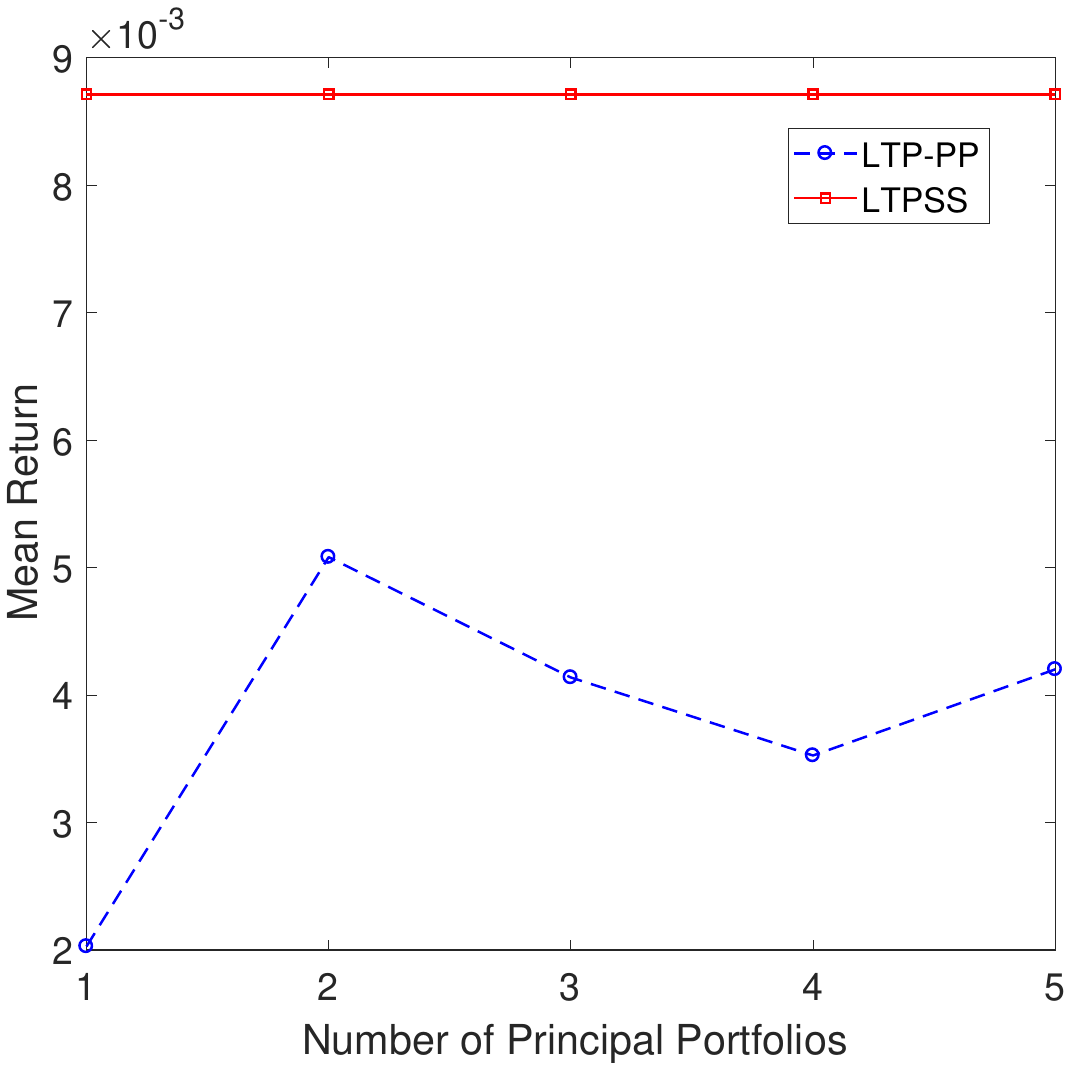}}
		\subfloat[Stoxx50]{
			\centering
			\includegraphics[width=0.33\textwidth]{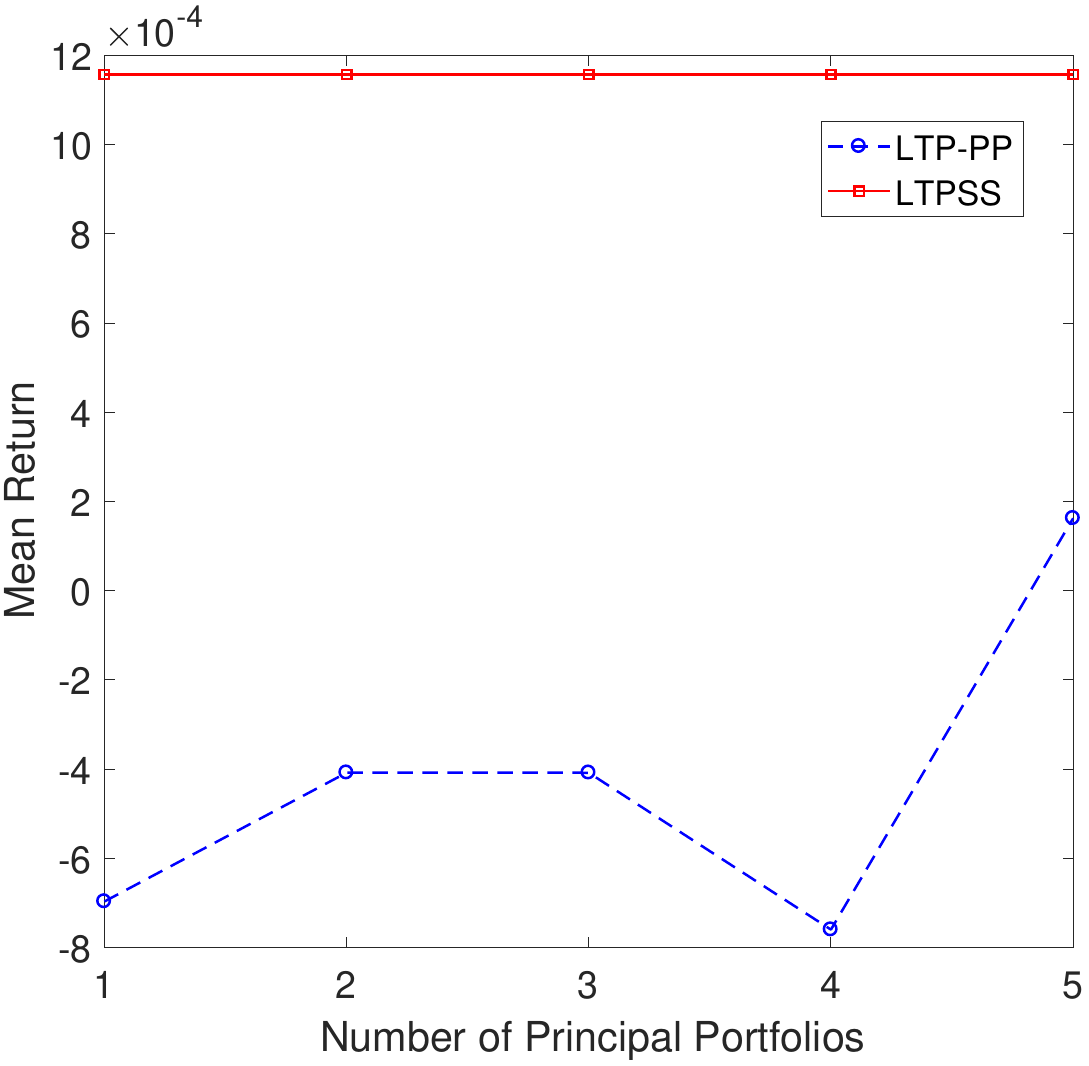}}
		\subfloat[FOF]{
			\centering
			\includegraphics[width=0.33\textwidth]{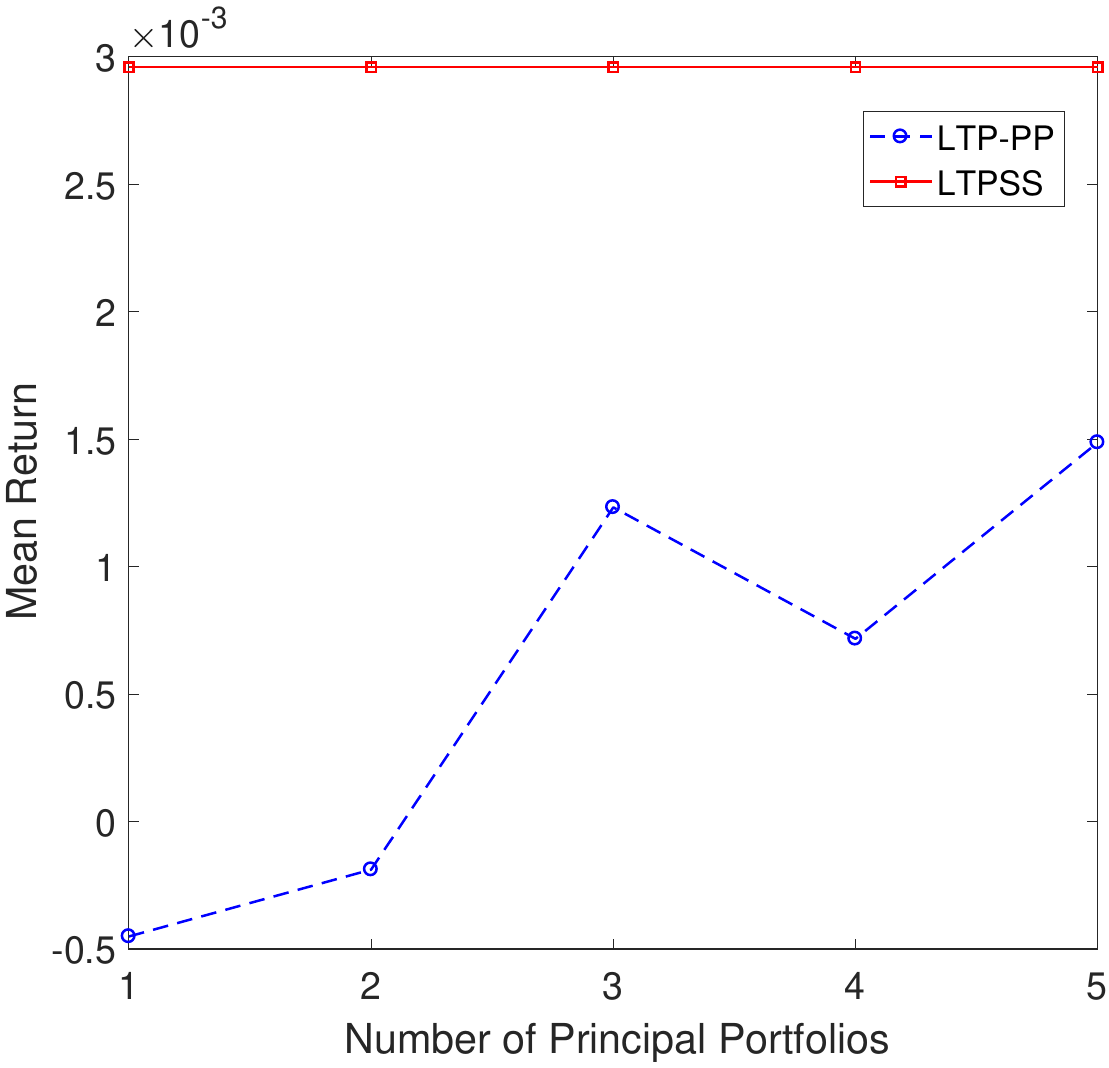}}\\
		\subfloat[FTSE100]{
			\centering
			\includegraphics[width=0.33\textwidth]{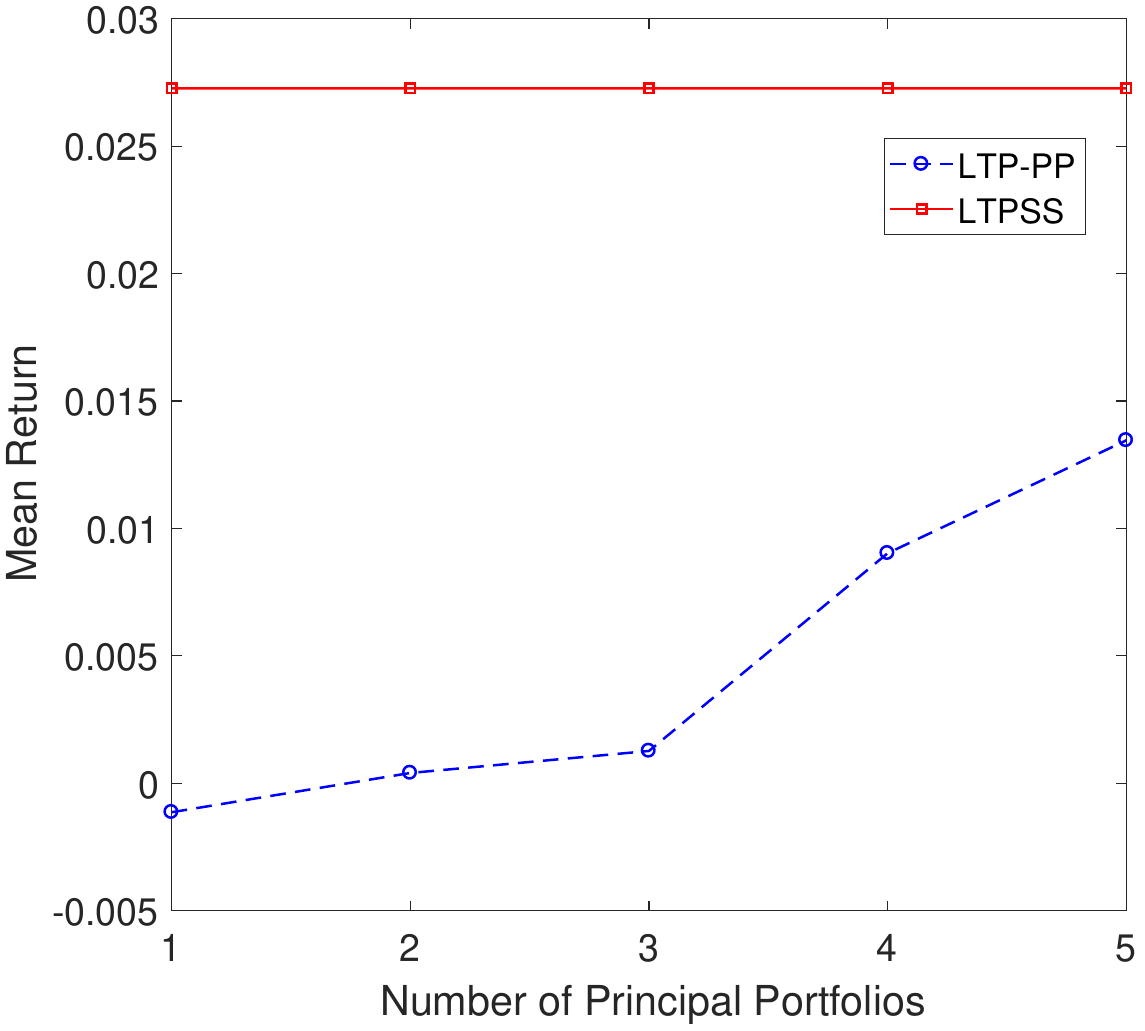}}	
		\caption{Mean returns for different numbers of principal portfolios on $7$ benchmark data sets.}
		\label{fig:MRmore}
	\end{figure*}

	\begin{figure*}[!htb]
		\centering
		\subfloat[FF25BM]{
			\centering
			\includegraphics[width=0.33\textwidth]{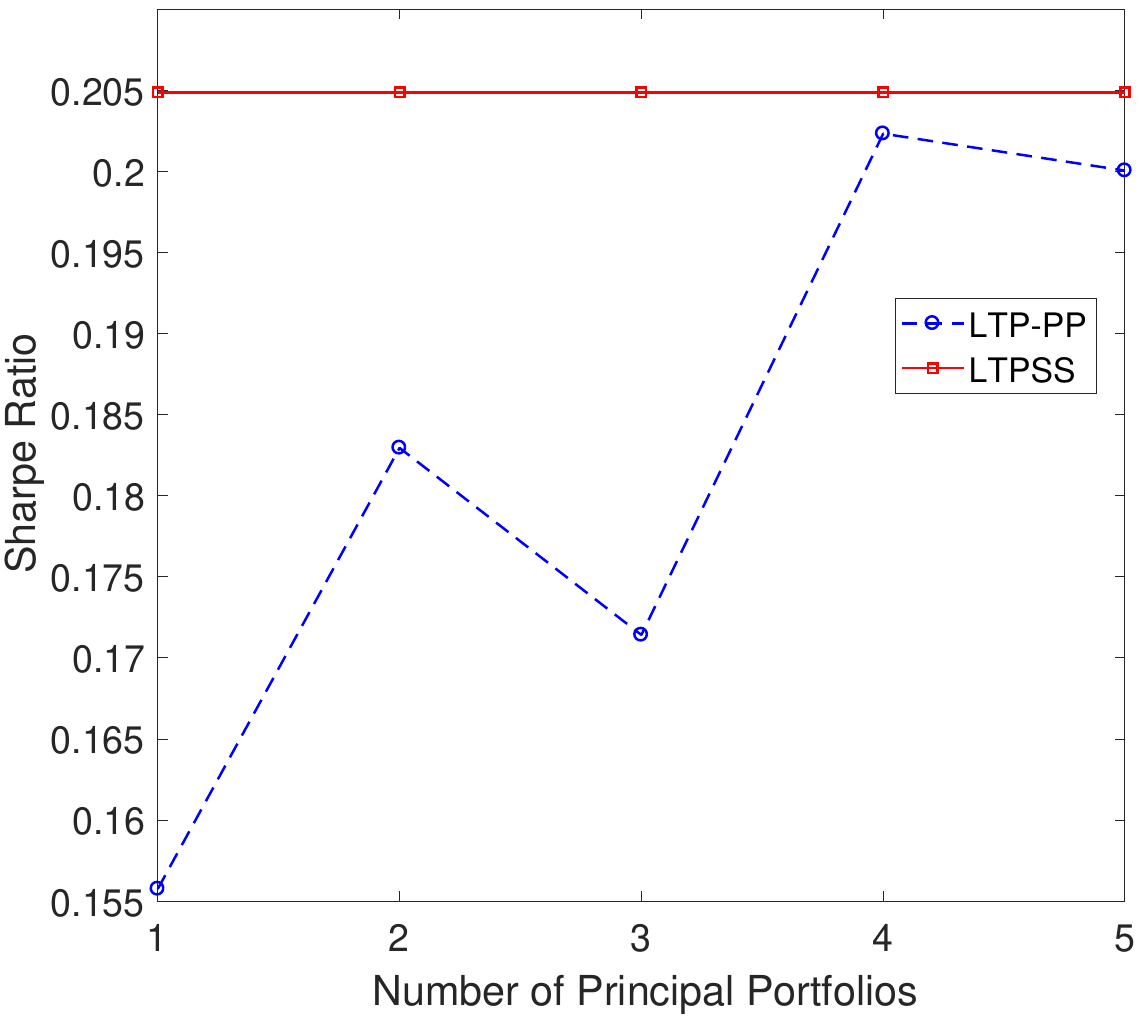}}
		\subfloat[FF25MEINV]{
			\centering
			\includegraphics[width=0.33\textwidth]{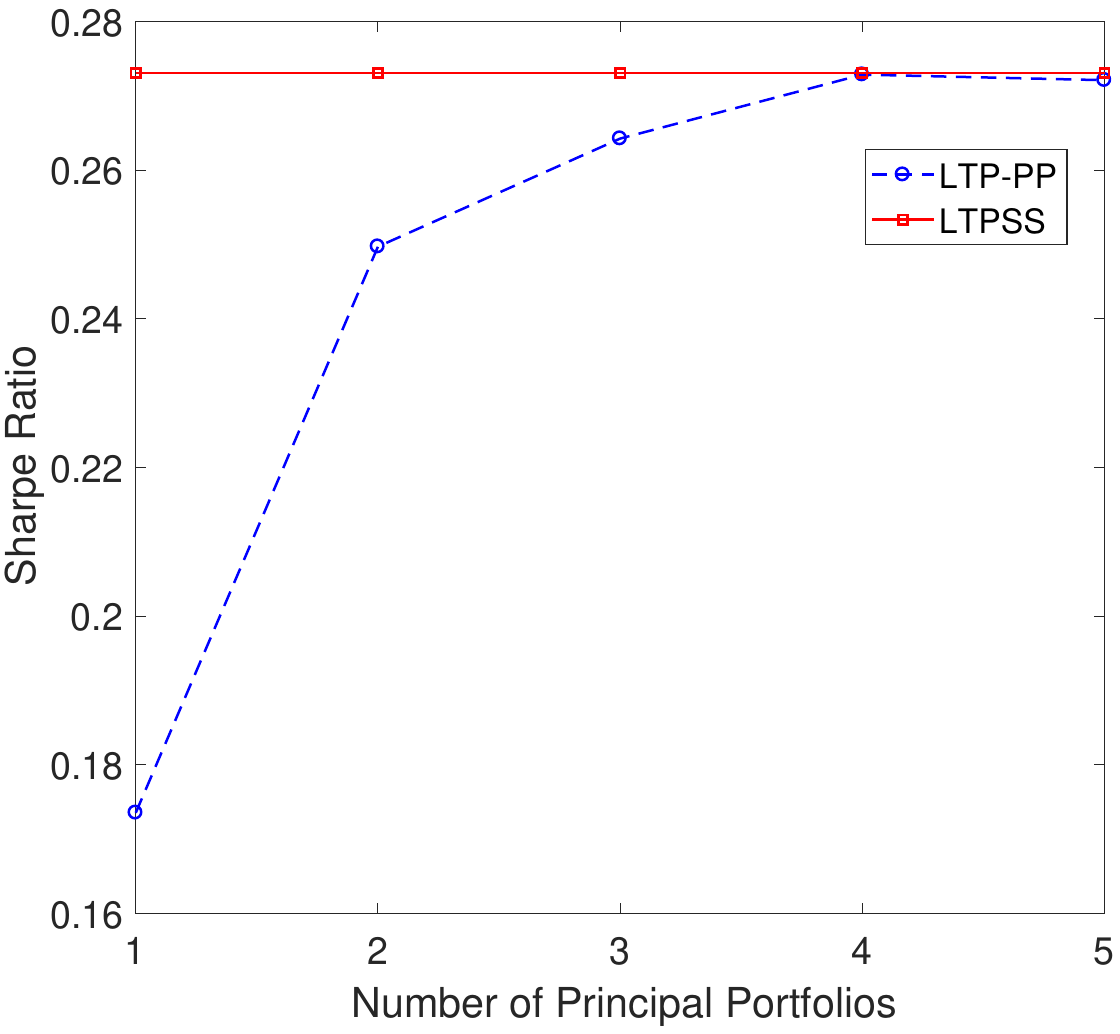}}
		\subfloat[FF25MEOP]{
			\centering
			\includegraphics[width=0.33\textwidth]{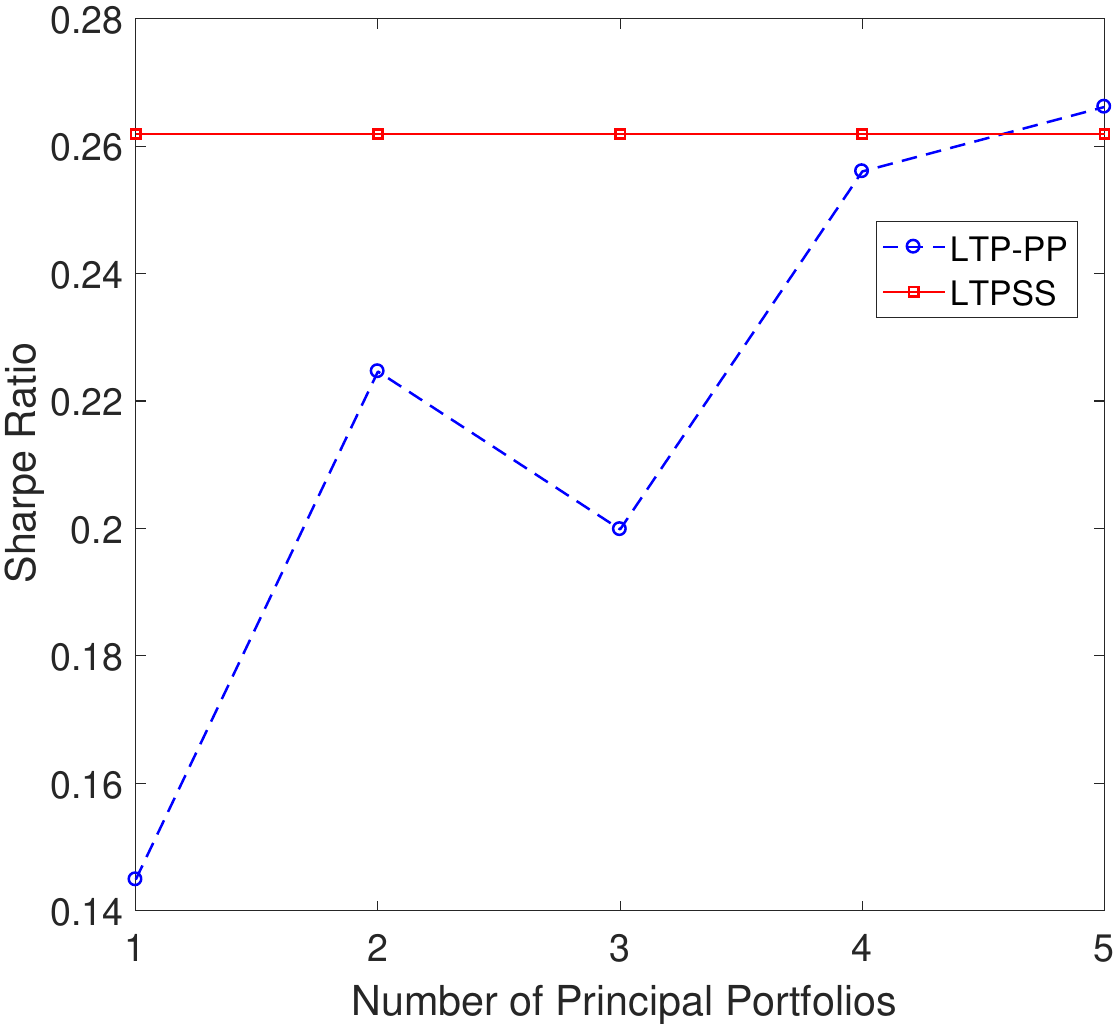}}\\
		\subfloat[MSCI]{
			\centering
			\includegraphics[width=0.33\textwidth]{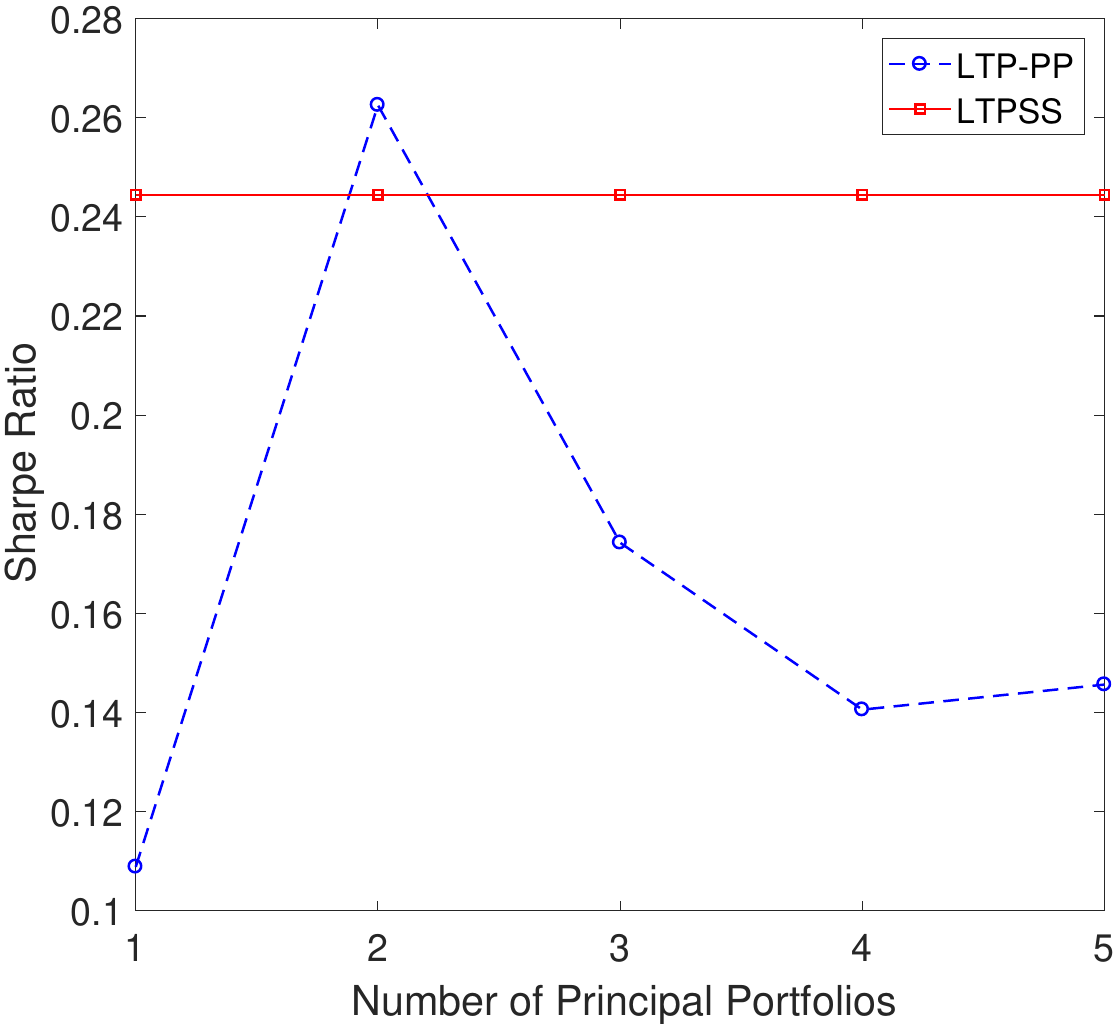}}
		\subfloat[Stoxx50]{
			\centering
			\includegraphics[width=0.33\textwidth]{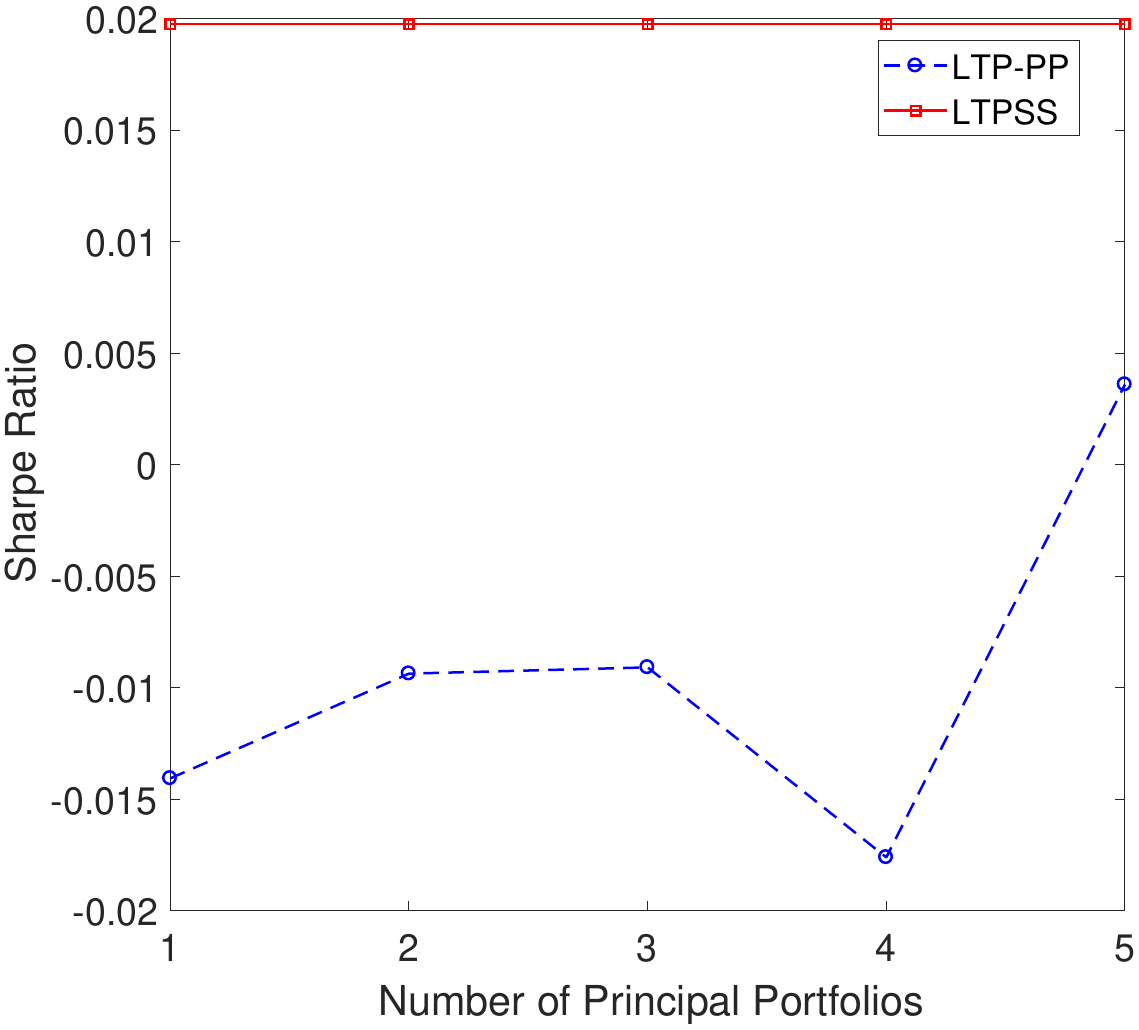}}
		\subfloat[FOF]{
			\centering
			\includegraphics[width=0.33\textwidth]{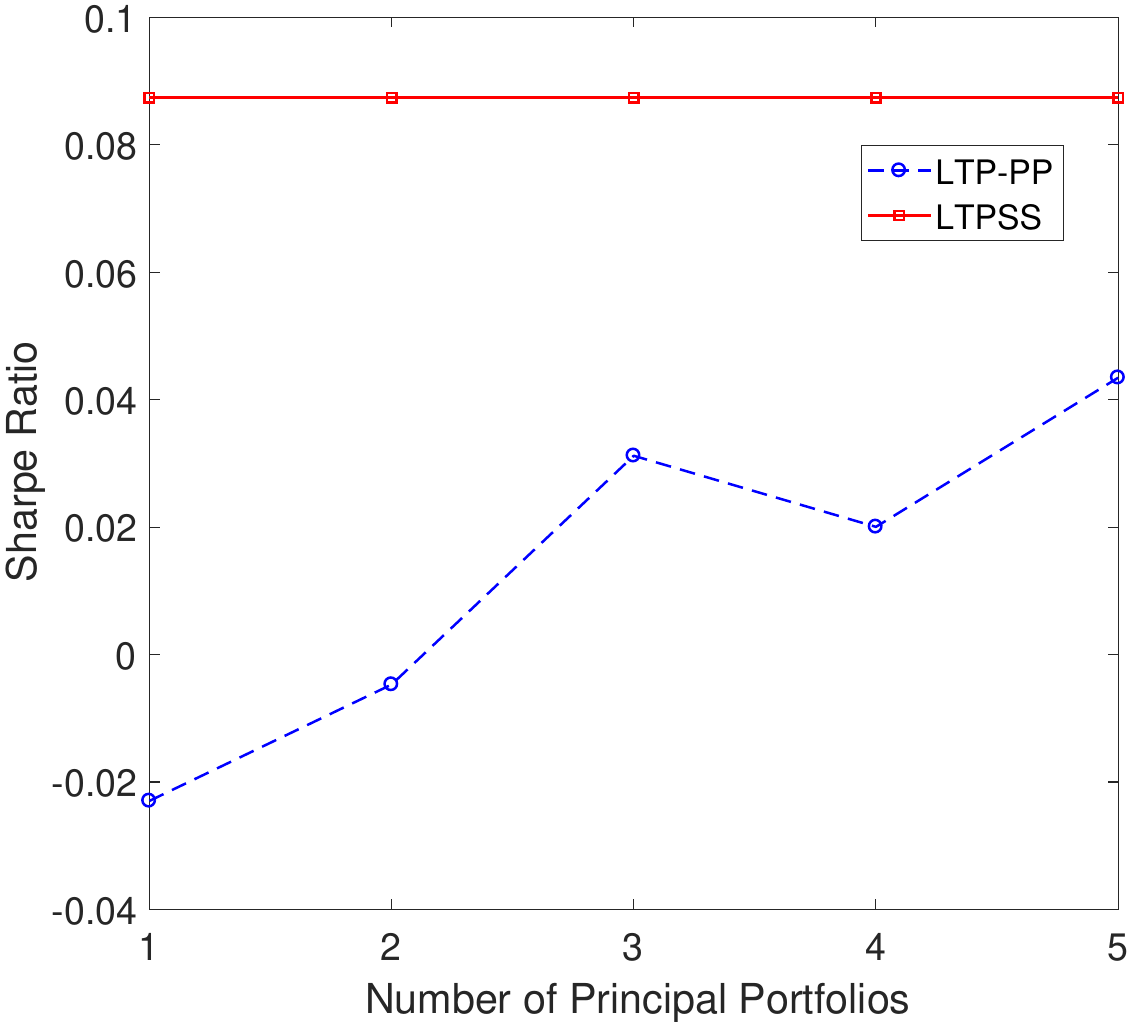}}\\
		\subfloat[FTSE100]{
			\centering
			\includegraphics[width=0.33\textwidth]{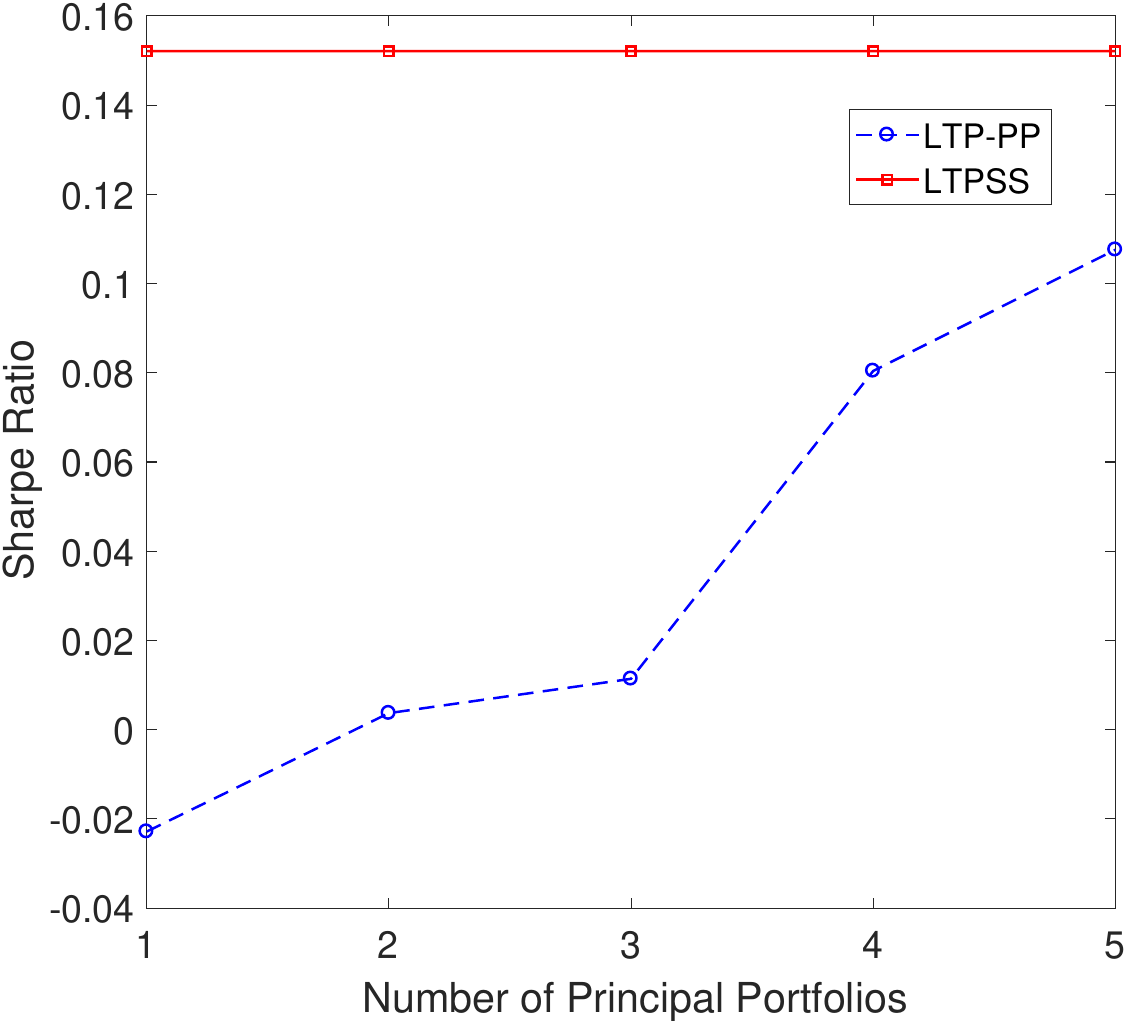}}	
		\caption{Sharpe ratios for different numbers of principal portfolios on $7$ benchmark data sets.}
		\label{fig:SRmore}
	\end{figure*}

	\begin{figure*}[!htb]
		\centering
		\subfloat[FF25BM]{
			\centering
			\includegraphics[width=0.33\textwidth]{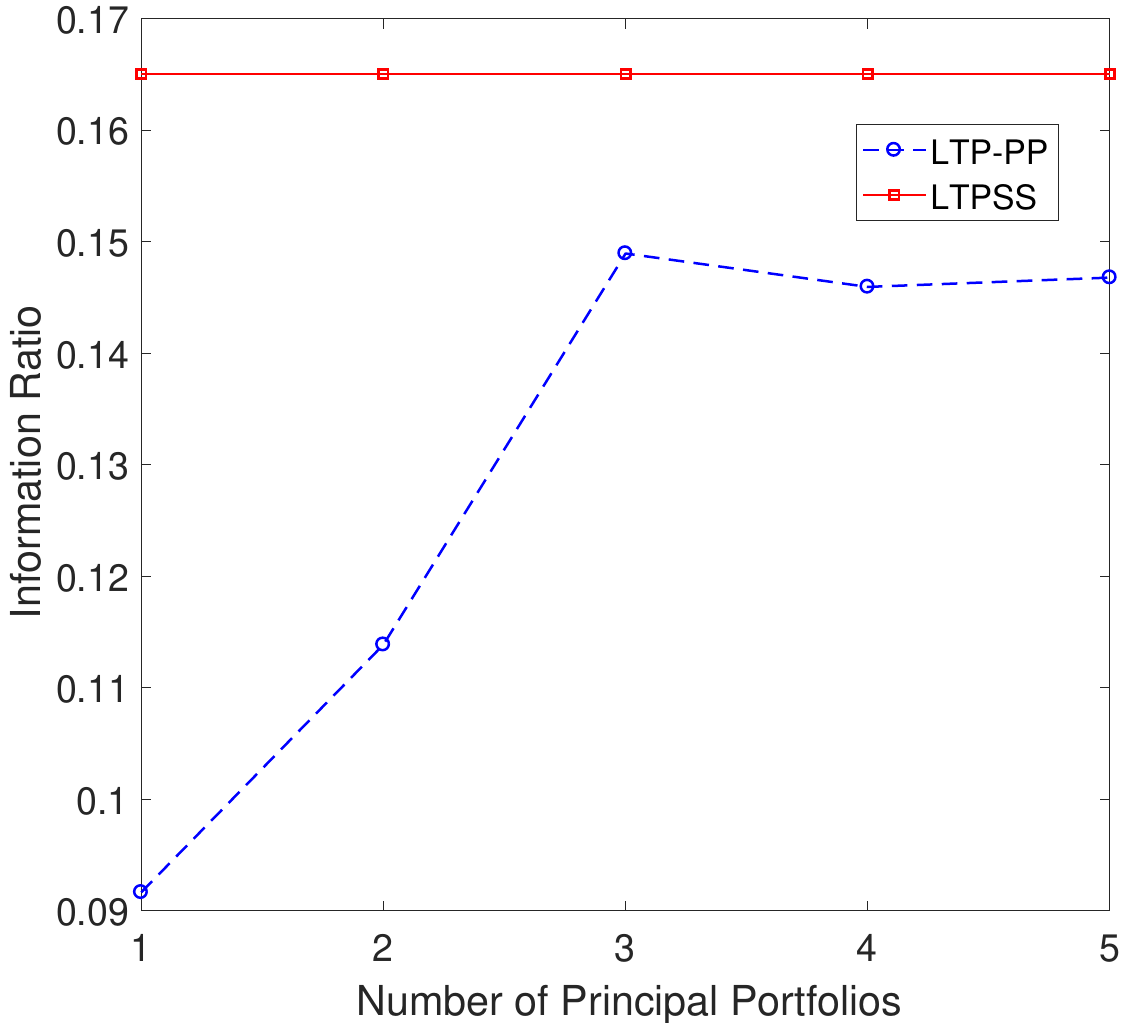}}
		\subfloat[FF25MEINV]{
			\centering
			\includegraphics[width=0.33\textwidth]{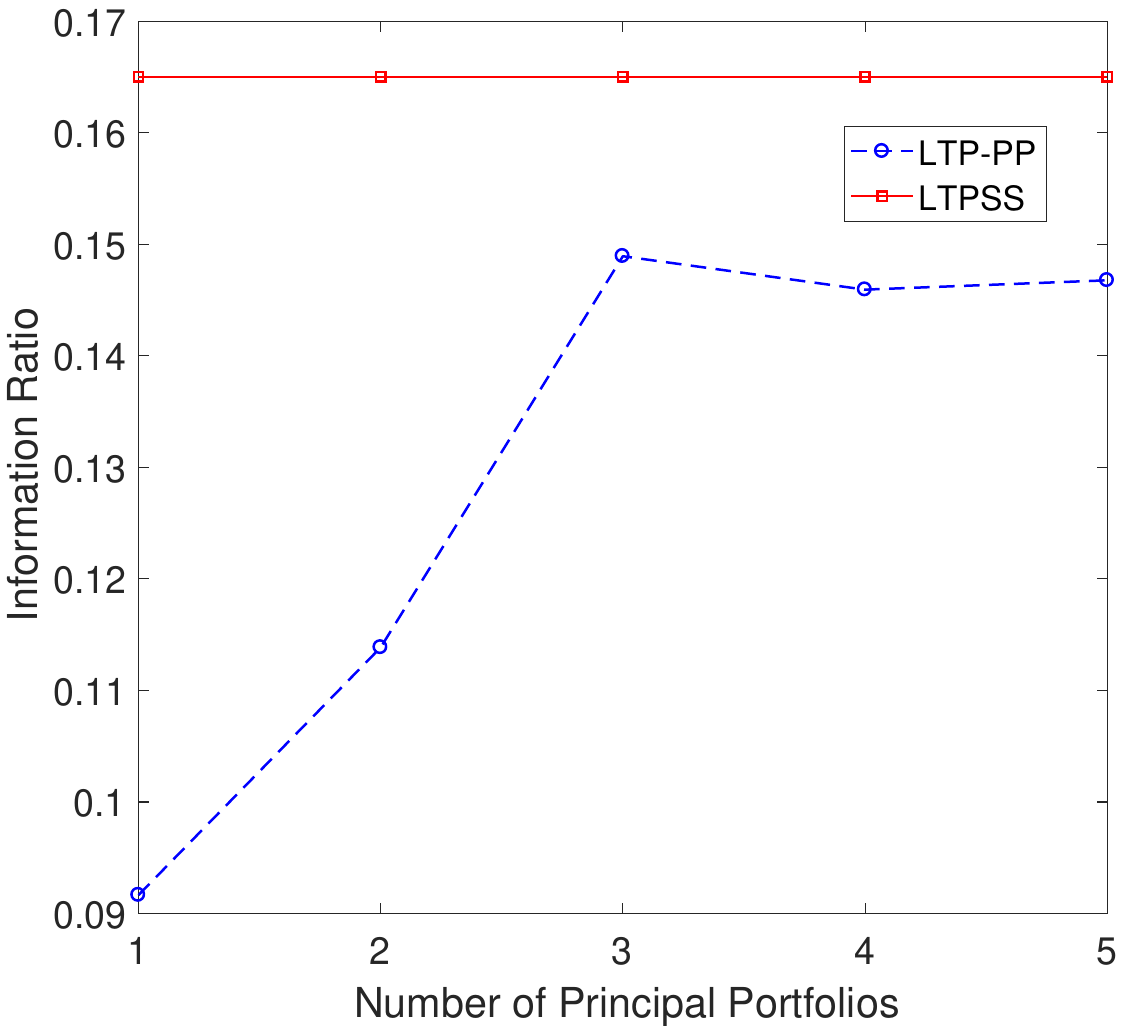}}
		\subfloat[FF25MEOP]{
			\centering
			\includegraphics[width=0.33\textwidth]{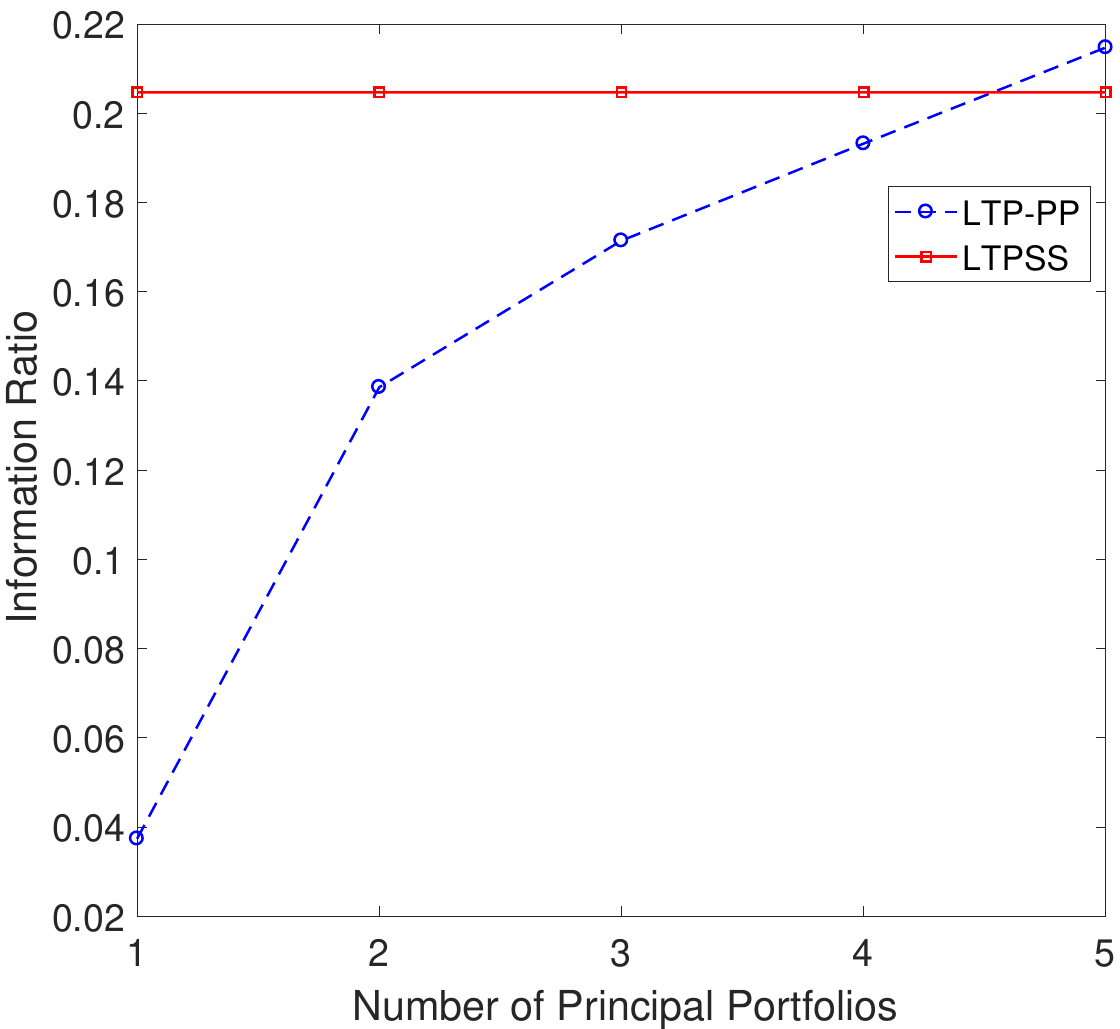}}\\
		\subfloat[MSCI]{
			\centering
			\includegraphics[width=0.33\textwidth]{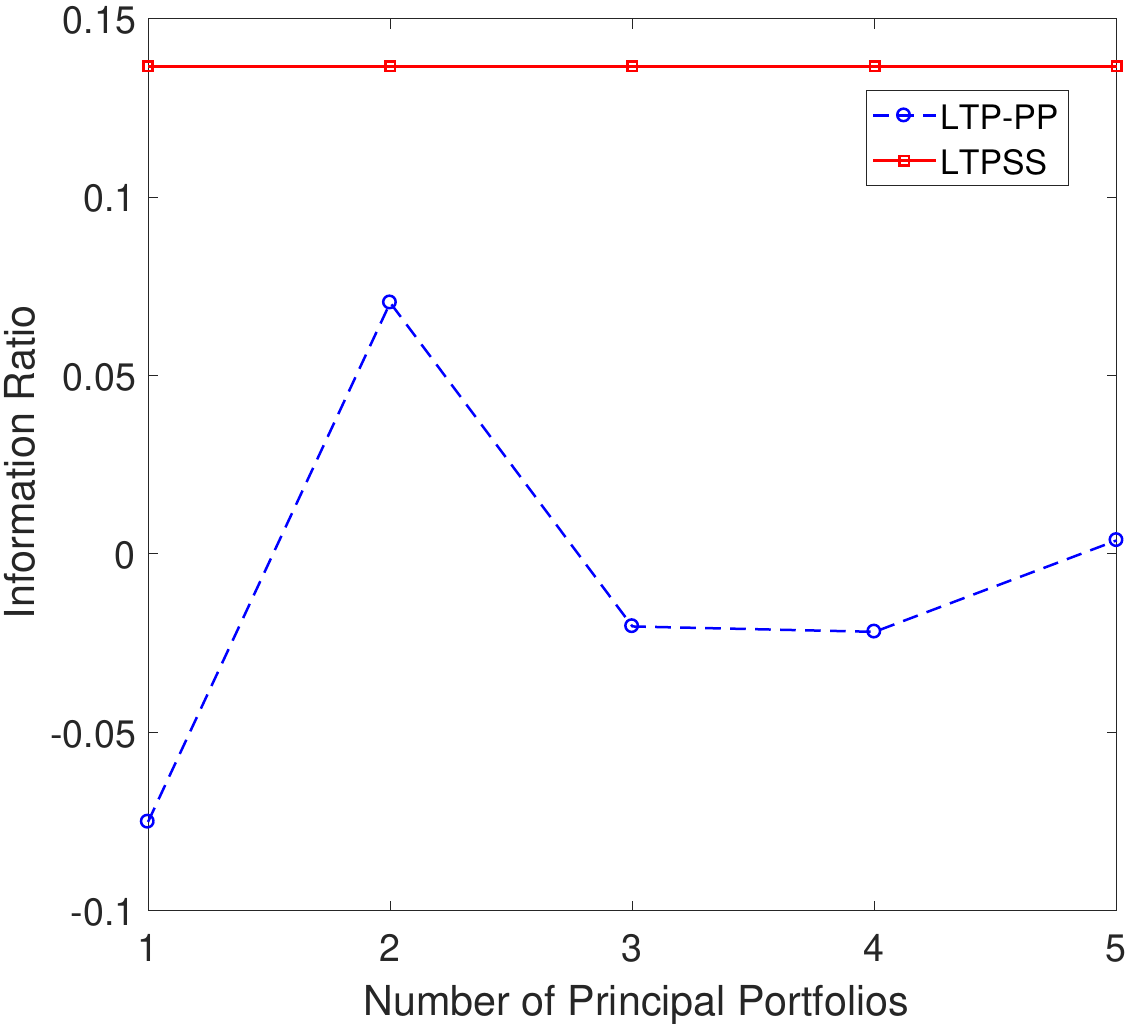}}
		\subfloat[Stoxx50]{
			\centering
			\includegraphics[width=0.33\textwidth]{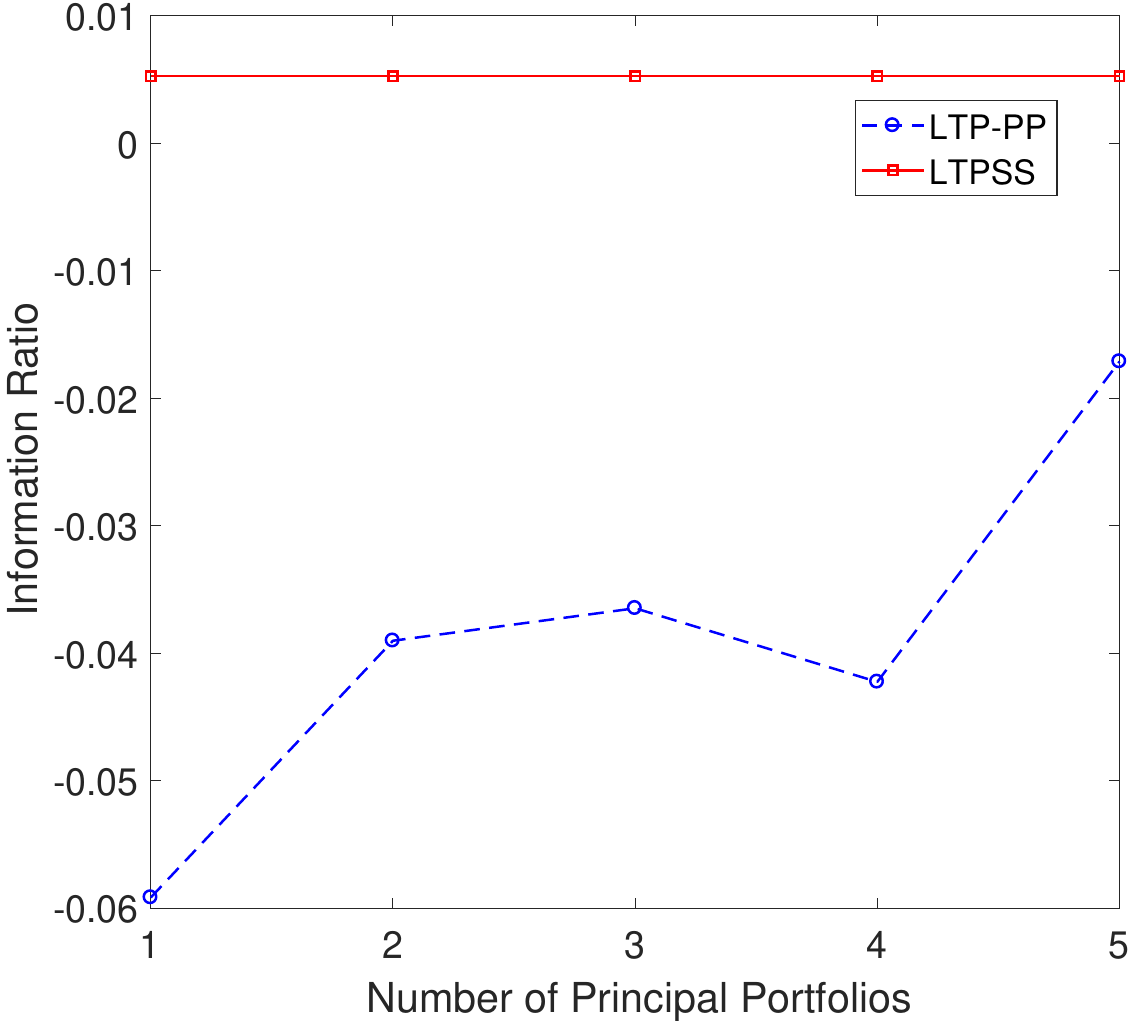}}
		\subfloat[FOF]{
			\centering
			\includegraphics[width=0.33\textwidth]{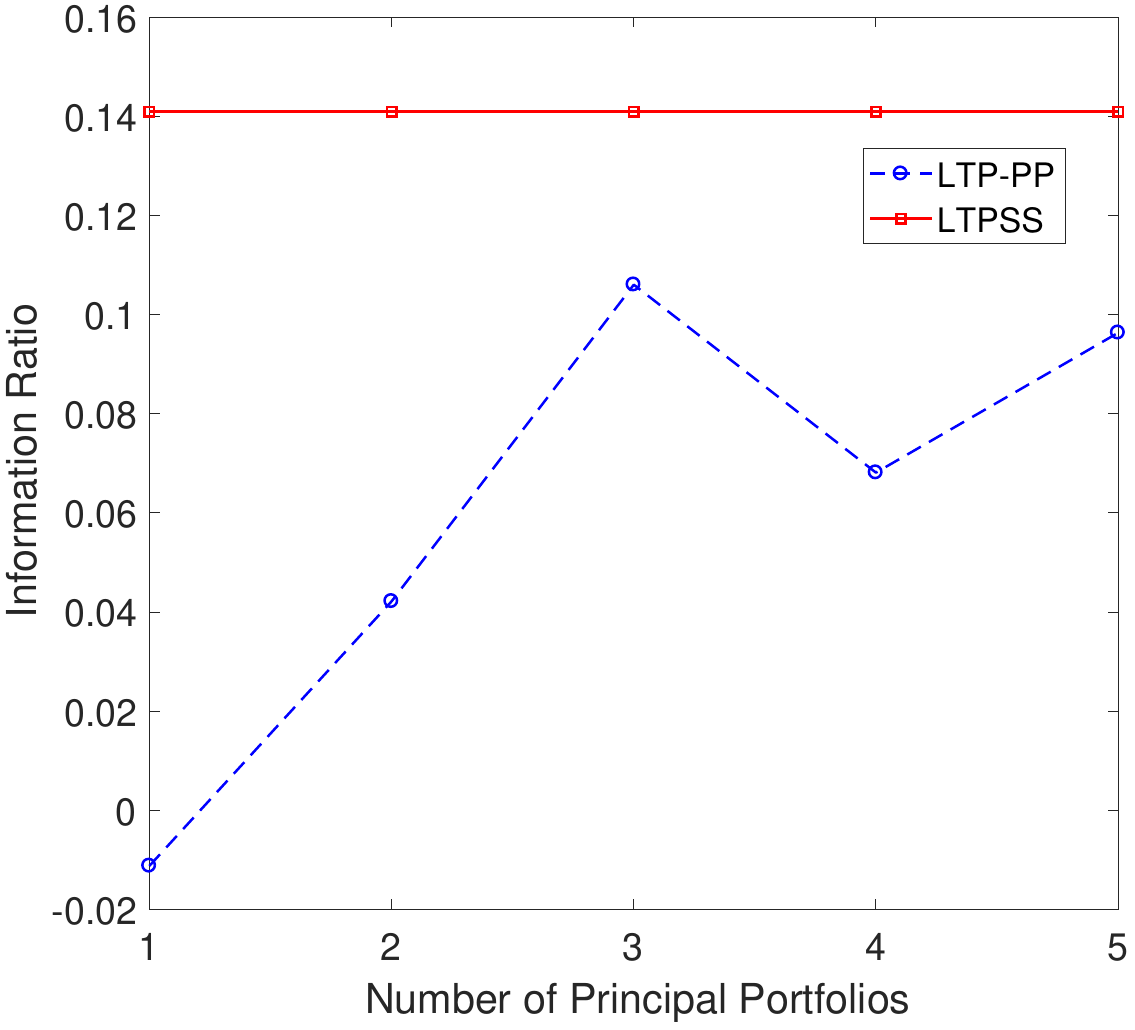}}\\
		\subfloat[FTSE100]{
			\centering
			\includegraphics[width=0.33\textwidth]{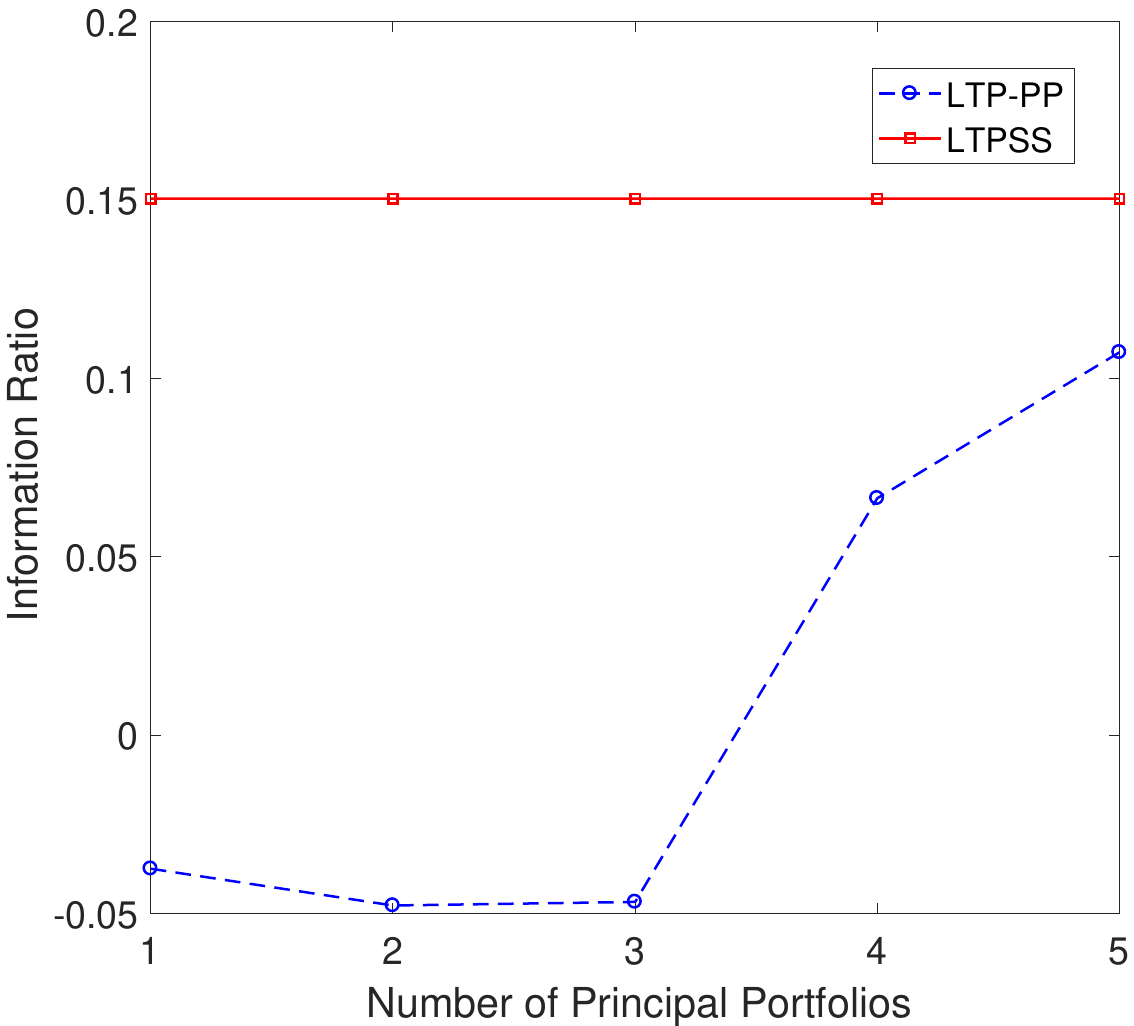}}	
		\caption{Information ratios for different numbers of principal portfolios on $7$ benchmark data sets.}
		\label{fig:IRmore}
	\end{figure*}

	\begin{figure*}[!htb]
		\centering
		\subfloat[FF25BM]{
			\centering
			\includegraphics[width=0.33\textwidth]{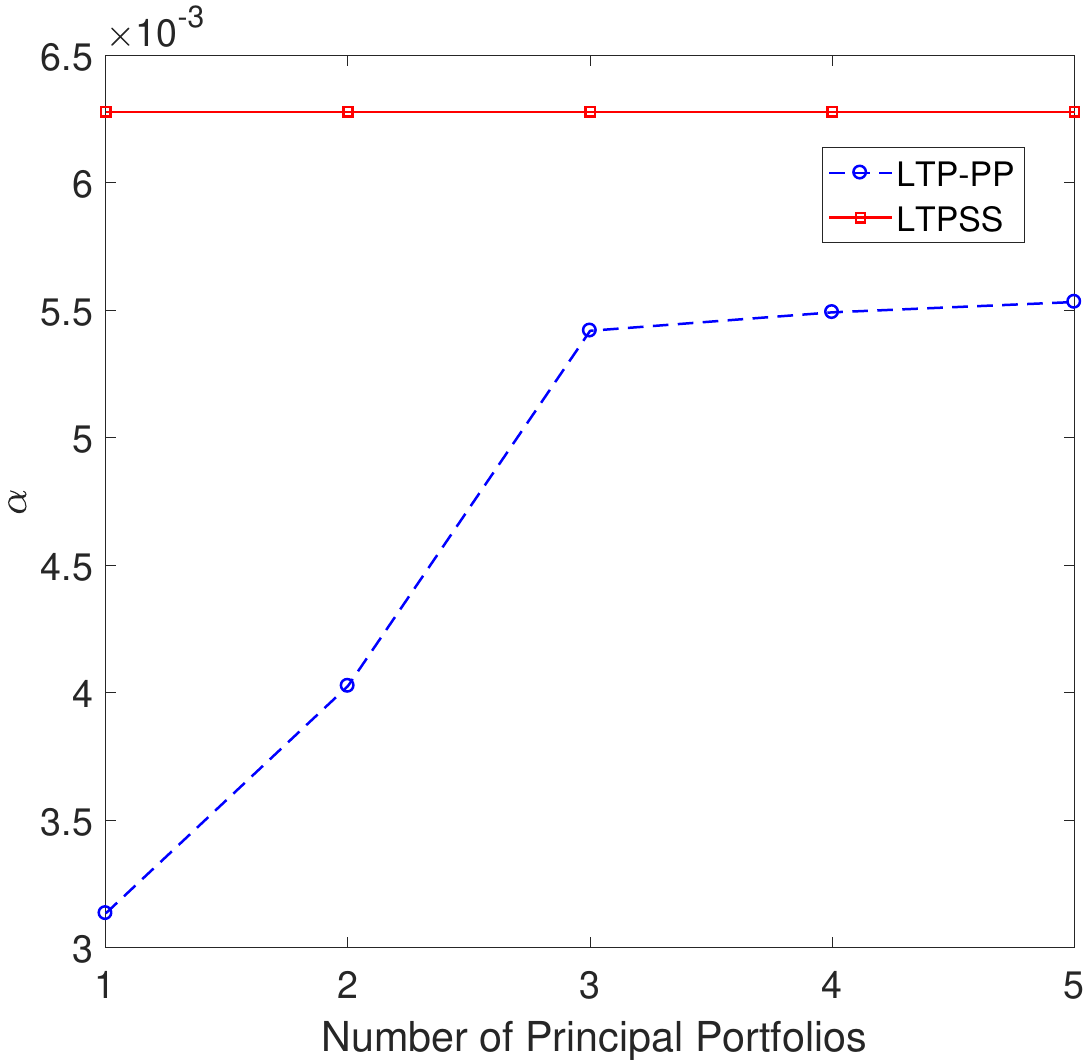}}
		\subfloat[FF25MEINV]{
			\centering
			\includegraphics[width=0.33\textwidth]{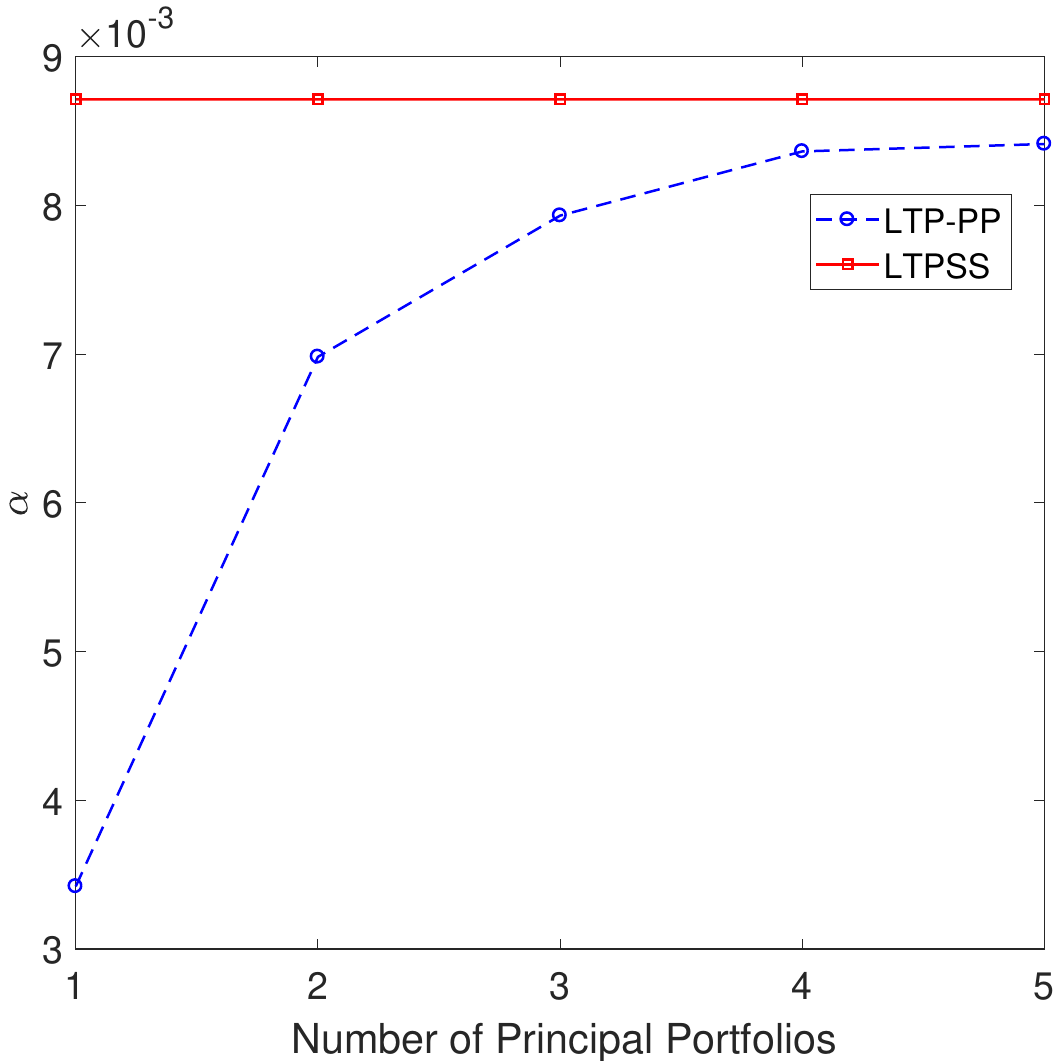}}
		\subfloat[FF25MEOP]{
			\centering
			\includegraphics[width=0.33\textwidth]{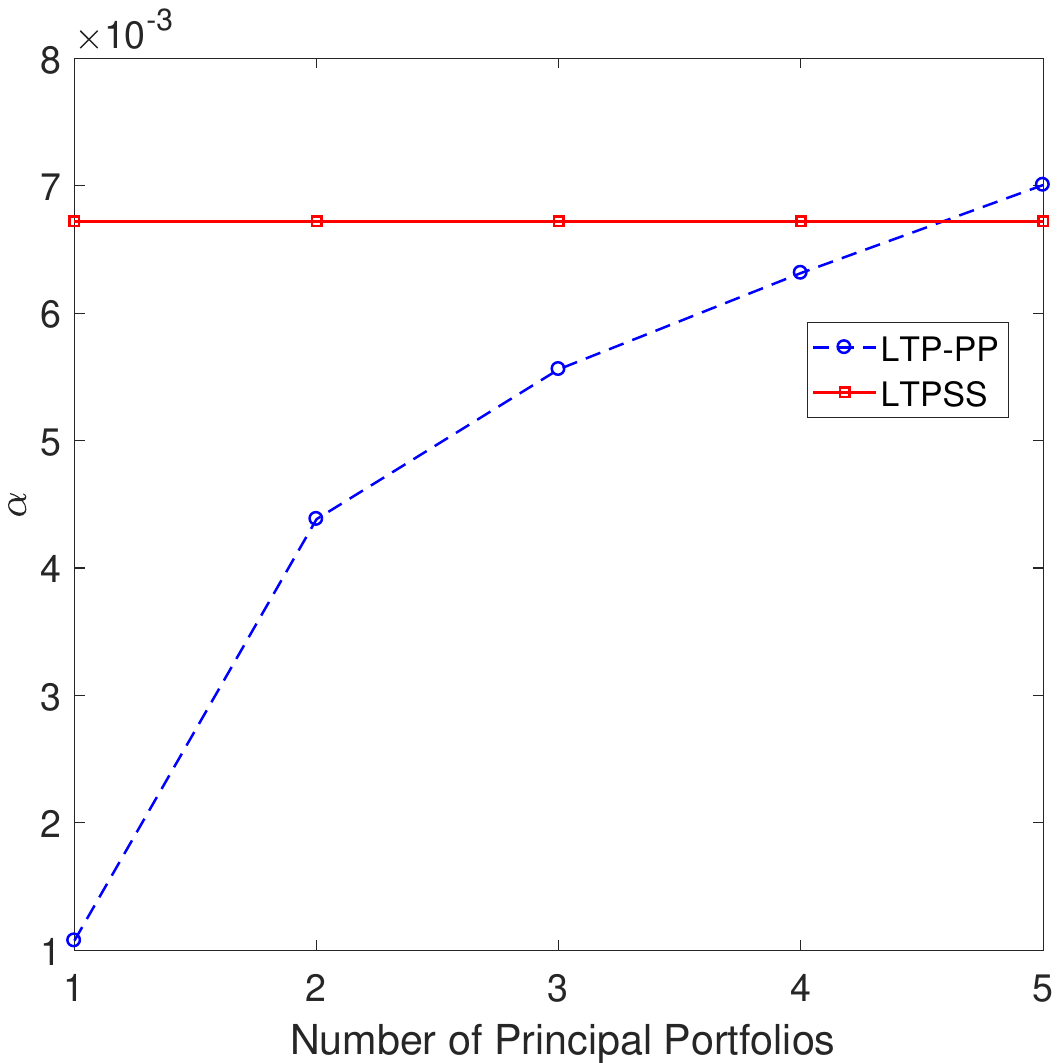}}\\
		\subfloat[MSCI]{
			\centering
			\includegraphics[width=0.33\textwidth]{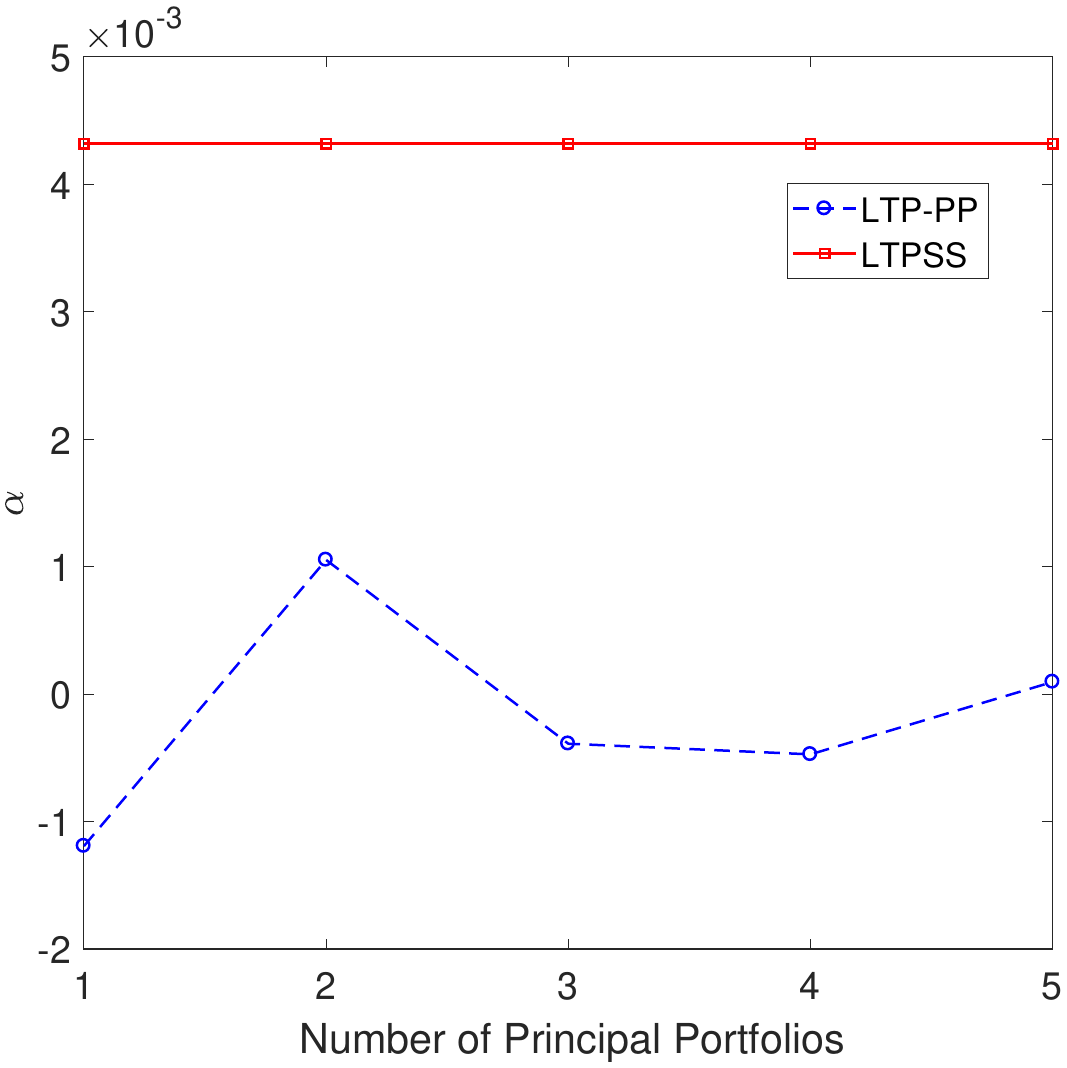}}
		\subfloat[Stoxx50]{
			\centering
			\includegraphics[width=0.33\textwidth]{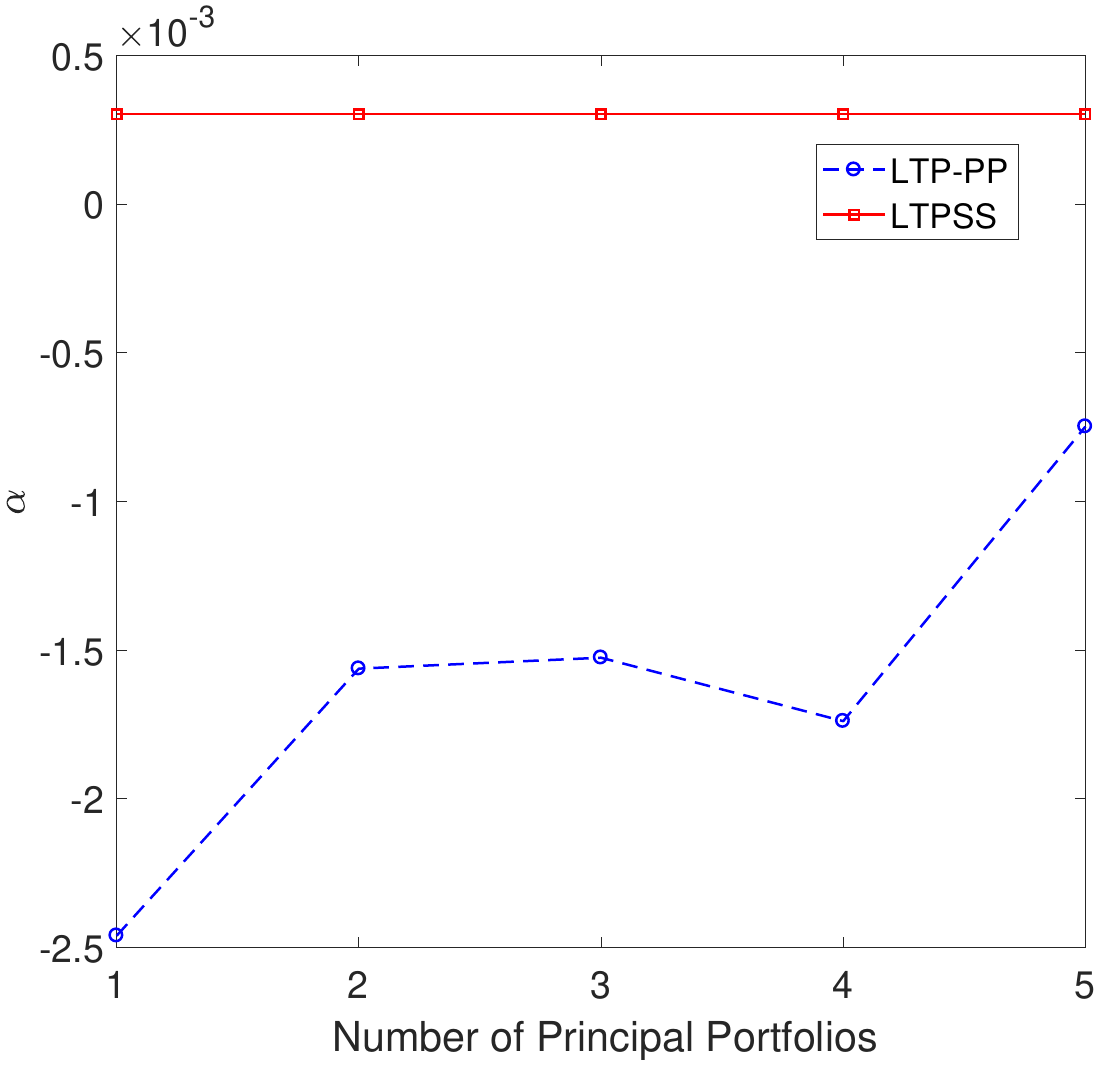}}
		\subfloat[FOF]{
			\centering
			\includegraphics[width=0.33\textwidth]{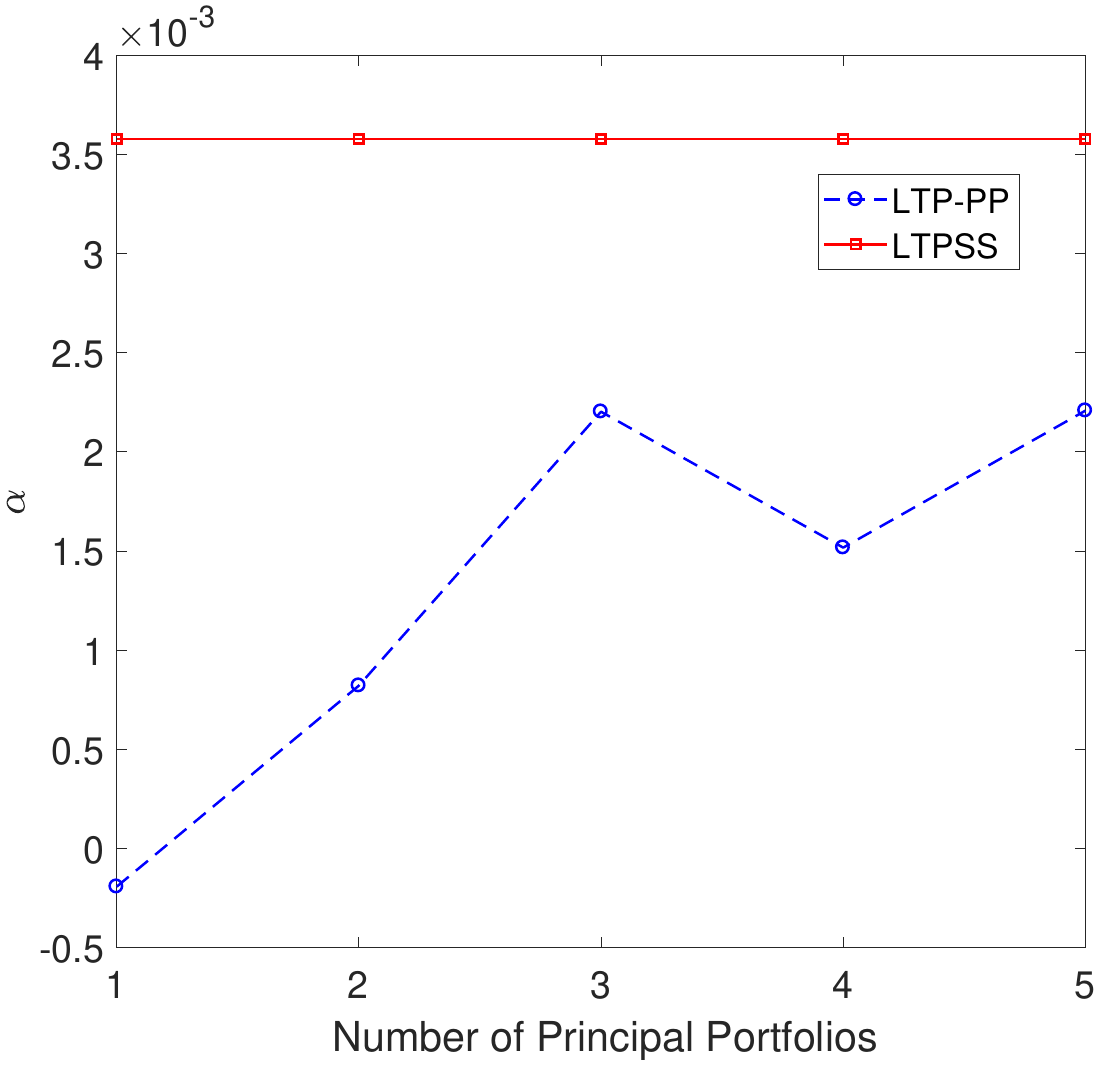}}\\
		\subfloat[FTSE100]{
			\centering
			\includegraphics[width=0.33\textwidth]{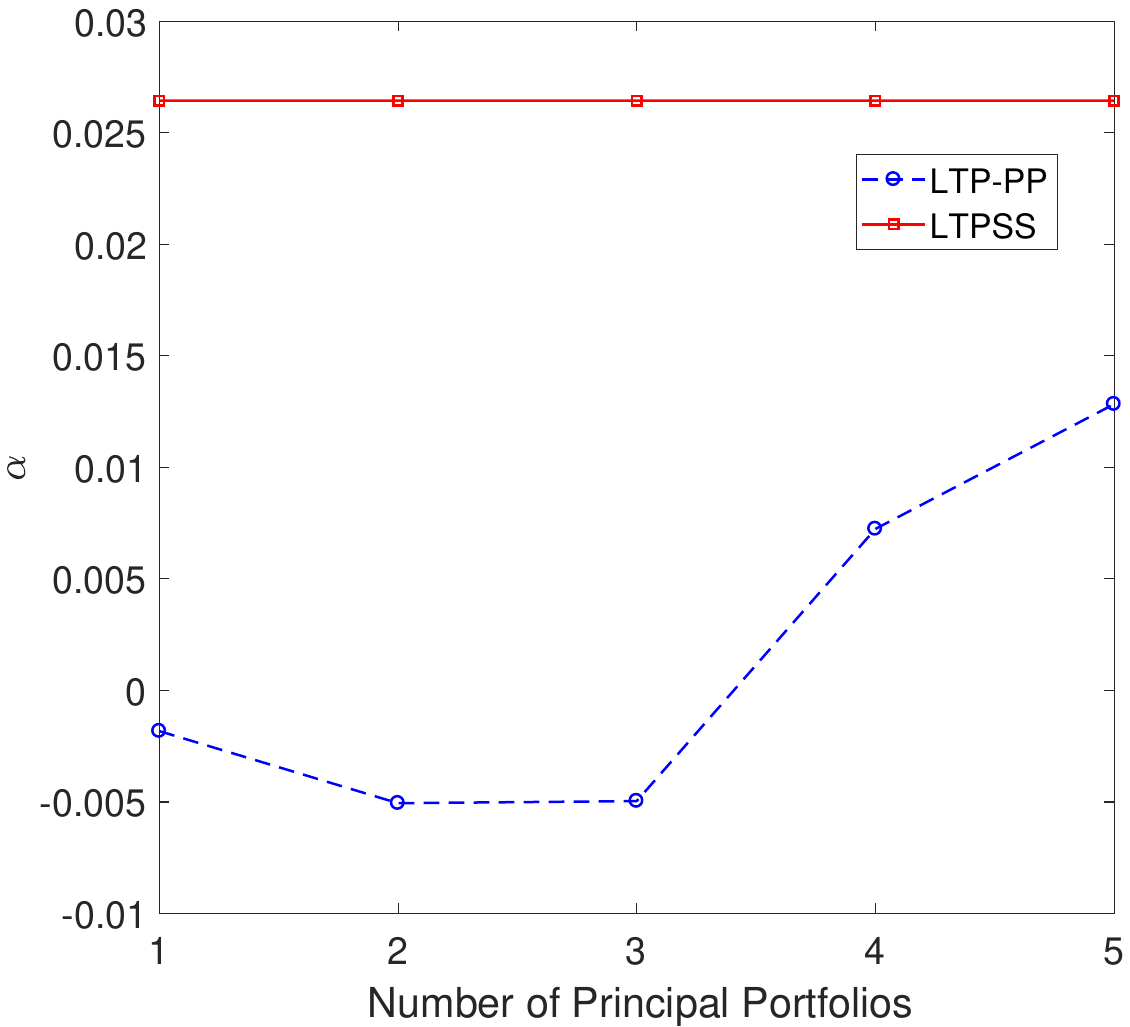}}	
		\caption{$\alpha$ factors for different numbers of principal portfolios on $7$ benchmark data sets.}
		\label{fig:alphamore}
	\end{figure*}

\end{document}